\documentclass[accepted]{uai2026} % after acceptance, for a revised version; 
\usepackage[american]{babel}
\usepackage{natbib} % has a nice set of citation styles and commands
\usepackage{mathtools} % amsmath with fixes and additions
\usepackage{booktabs} % commands to create good-looking tables
\usepackage{tikz} % nice language for creating drawings and diagrams

\usepackage{hyperref}       % hyperlinks
\usepackage{url}            % simple URL typesetting
\usepackage{booktabs}       % professional-quality tables
\usepackage{amsfonts}       % blackboard math symbols
\usepackage{nicefrac}       % compact symbols for 1/2, etc.
\usepackage{xcolor}         % colors

\usepackage{amsmath}
\usepackage{amssymb}
\usepackage{mathtools}
\usepackage{amsthm}
\usepackage{graphicx}
\usepackage{caption}
\usepackage{subcaption}
\usepackage{bm}
\usepackage{color}

\DeclareMathOperator{\im}{im}
\usepackage[capitalize,noabbrev]{cleveref}

\theoremstyle{plain}
\newtheorem{theorem}{Theorem}[section]
\newtheorem{proposition}[theorem]{Proposition}
\newtheorem{lemma}[theorem]{Lemma}

\theoremstyle{definition}

\newtheorem{assumption}[theorem]{Assumption}
\theoremstyle{remark}

\title{Towards Identifiability of Interventional Stochastic Differential Equations}

\author[1,2,7]{Aaron Zweig}{}
\author[1,3,4,5]{Zaikang Lin}
\author[2,6,7]{Elham Azizi}
\author[1,2]{David A. Knowles}
\affil[1]{%
    New York Genome Center\\
    New York, U.S.
  }
\affil[2]{%
    Department of Computer Science\\
    Columbia University\\
    New York, U.S.
}
\affil[3]{%
    Department of Applied Mathematics and Applied Physics\\
    Columbia University\\
    New York, U.S.
}
\affil[4]{%
Institute of Computational Biology\\
Helmholtz Munich\\
Munich, Germany
}

\affil[5]{%
Department of Mathematics\\
Technische Universität München\\
Munich, Germany
}

\affil[6]{%
    Department of Biomedical Engineering\\
    Columbia University\\
    New York, U.S
}
\affil[7]{%
    Irving Institute for Cancer Dynamics, Columbia University
}
  
\begin{document}
\maketitle

\begin{abstract}
We study identifiability of stochastic differential equations (SDE) under multiple interventions.  Our results give the first provable bounds for unique recovery of SDE parameters given samples from their stationary distributions. We give tight bounds on the number of necessary interventions for linear SDEs, and upper bounds for nonlinear SDEs in the small noise regime.  We experimentally validate the recovery of true parameters in synthetic data, and motivated by our theoretical results, demonstrate the advantage of parameterizations with learnable activation functions in application to gene regulatory dynamics.
\end{abstract}

\section{Introduction}

Stochastic dynamical systems are ubiquitous as models for natural data.  They are perfectly suited for application to time-series data, and therefore also a good candidate to characterize systems that reach a steady state in the limit. If a system is governed by some stochastic differential equation (SDE) and the same system is observed under different interventions, ideally one would learn the underlying parameters governing the dynamics, and guarantee accurate prediction under new interventions.

However, in many natural settings, data is modeled as following an SDE even if one does not have access to explicit trajectories. Studies of ecological systems focus on the long-term survival of multiple species modeled by the quasi-stationary state of SDEs with environmental factors as perturbations~\citep{hening2021stationary}.  The application of flow cytometry to protein signaling networks under perturbation~\citep{sachs2005causal} is destructive and yields protein quantification at one time point, modeled using the stationary distributions of linear SDEs in~\citet{varando2020graphical}.

One highly motivating application is single-cell genomic sequencing with high-throughput CRISPR perturbations.  Biologists are often interested in inferring the gene regulatory network (GRN) that characterizes the dynamics of gene expression, informing which genes should be targeted for treatment~\citep{dixit2016perturb}.  But the destructive nature of sequencing makes it impossible to observe the trajectory of a single cell at multiple time-points, and in general it is difficult to obtain any time-series genomics data due to the high expense.  Therefore, practitioners often only collect data at the end of an experiment, i.e., from the stationary distribution of the system.

% High-throughput perturbations, e.g., using CRISPR systems, coupled with single-cell genomic sequencing provide the opportunity to study the dynamic and causal relationships (known as the ``gene regulatory network'', GRN) between genes at an unprecedented resolution, under a wide set of possible interventions, informing which genes should be targeted for treatment~\citep{dixit2016perturb}.  

% Due to the high degree of gene co-expression, substantial noise and unmeasured modalities (e.g., chromatin state, protein expression), latent confounding is particularly severe in genomic data, leading to considerable false positive rates among popular methods~\citep{kernfeld2024transcriptome}.  

Understanding the dynamics is essential for extrapolating to unseen settings, but noise and latent confounding makes it non-trivial to determine the true dynamics.  Causal disentanglement aims to learn causal factors in spite of these confounders, mainly focusing on directed acyclic graph (DAG) based methods.  To demonstrate these methods are well-founded, there is considerable effort devoted to understanding which models have identifiability guarantees~\citep{lachapelle2022disentanglement}.  However, these models suffer from inherent weakness, in particular 1) being unable to represent cycles or 2) approximate continuous-time dynamical models.

There has been renewed interest in modeling with stochastic differential equations (SDE) directly~\citep{peters2022causal}.  In the genomic context, there is precedent for this type of modeling to represent the so- called ``Waddington landscape''~\citep{waddington2014strategy}, the hypothetical energy surface of cells.  Furthermore, SDEs are commonly used for simulating transcriptomic datasets from a given gene regulatory network~\citep{pratapa2020benchmarking,dibaeinia2020sergio}.

As demonstrated in the context of diffusion models, SDEs are fully expressive in practice and can accurately generate observational data~\citep{song2021maximum}.  But to identify the true underlying SDE requires assumptions on the model.  Foundational theoretical works on identifying dynamical systems typically learn from many trajectories or even the infinitesimal generator of the dynamics~\citep{hansen2014causal}, leaving open the harder setting of observing only the stationary distribution.

% But the destructive nature of sequencing makes it impossible to observe the trajectory of a single cell at multiple time-points, and in general it is difficult to obtain any time-series genomics data due to the high expense.  Therefore, practitioners often only collect data at the end of an experiment, i.e., from the stationary distribution of the system.  Consequently, several dynamics-based models rely on single time points for evaluation, and this lack of trajectories may render the true biological model unidentifiable.

Our interest in this work is to verify which parametric assumptions are necessary for dynamical systems to have identifiability guarantees without trajectory data.  Namely:
\begin{center}
    \textit{How many interventions are necessary to identify the parameters of a stochastic differential equation, only given access to the stationary distribution?}
\end{center}

\paragraph{Contributions} In this work, we offer the first analysis of identifiability of interventional stochastic differential equations, with data restricted to the stationary measure.  Specifically:

\begin{itemize}
    \item We characterize tight bounds on the number of interventions necessary for identifiability of linear SDEs with shift interventions.
    \item We extend this analysis to nonlinear SDEs in the small noise regime, showing that identifiability is possible even without knowing the activation function of the true model.
    \item We apply this insight to synthetic data and semi-synthetic genomic data to confirm the efficacy of learned activations in causal SDEs, which improve expressiveness without sacrificing a simple structure and enable the inference of gene regulatory networks.
\end{itemize}
%%%%%%%%%%%%%%%%%%%%%%%%%%%%%%%%
\section{Setup}

\subsection{Notation}

We will write the elementary basis vectors as $\{e_i\}_i$.  We will consider $\sigma: \mathbb{R}^n \rightarrow \mathbb{R}^n$ as any elementwise function, i.e. $\sigma_i(x) = \sigma_i(x_i)$.  This includes elementwise activations as they are applied in multilayer perceptrons (MLPs), but we also allow for elementwise functions where each component acts differently.  Writing $\sigma'$ will, unless otherwise described, denote the map $\mathbb{R}^n \rightarrow \mathbb{R}^n$ that applies elementwise the derivative of each component function, i.e. $\sigma'(x) = [\sigma_1'(x_1), \dots, \sigma_n'(x_n)]$.  We will use $Jf$ to denote the Jacobian of a vector valued function $f$, and $\Delta f$ to denote the Laplacian of $f$.  We will use the notation $diag$ to map vectors to diagonal matrices, or to map matrices to exclusively their diagonal part.

We use $s_i(A)$ to denote the $i$th singular value of matrix $A$.  We let $P_A$ denote the orthogonal projection onto the image of matrix $A$, and $P_A^\perp = I - P_A$ the orthogonal projection onto its complement.  We let $A^\dag$ denote the pseudoinverse of $A$, and note that if $A$ has linearly independent columns it is a left inverse such that $A^\dag A = I$, likewise for rows and right inverse.  We let $\|\cdot\|$ denote the spectral norm and $\|\cdot\|_F$ the Frobenius norm.  We also introduce a minimal signed permutation distance $d_P(X, Y) = \min_{\Lambda, \Pi} \|X - Y\Pi \Lambda\|$ where $\Lambda$ is minimized over diagonal matrices with entries in $\pm 1$ and $\Pi$ is minimized over permutation matrices.

We will use $\lesssim$ to denote inequality up to constant factor (treating some problem parameters as constants where specified), and similarly reserve the notation $\tilde{C}$ for any absolute constant greater than zero.  Finally iid indicates independent and identically distributed.

\subsection{Stochastic Differential Equations}

We consider SDEs of the following form, where $v$ is a vector field and $B_t$ is standard Brownian motion:
\begin{align}
    dX_t = v(X_t)dt + \sqrt{\epsilon}dB_t.
\end{align}
We will only consider autonomous systems, i.e. where the drift and noise terms have no dependence on time $t$.  We will enforce the weak conditions on the drift and noise to guarantee a unique stationary distribution~\citep{berglund2021long}, with a density $p$ that satisfies the Fokker-Planck equation,
\begin{align}
    0 = -\nabla \cdot (pv) + \frac{\epsilon}{2} \Delta p.
\end{align}

\subsection{Linear SDEs}

We need some classical facts about linear SDEs, which are better understood than their nonlinear cousins, since Fokker-Planck can be solved explicitly.

\begin{theorem}[\citet{sarkka2019applied}]\label{thm:sarkka-linear}
    Given square matrices $L, Q$ and vector $c$, consider the SDE
    \begin{align}
        dX_t = (LX_t + c) dt + Q dB_t 
    \end{align}
    Assume $L$ is Hurwitz, i.e., all its eigenvalues have strictly negative real parts, and $Q$ is full rank. Then the unique stationary distribution is $\mathcal{N}(-L^{-1}c, \omega)$ where $\omega$ is the unique solution to the Lyapunov equation,
    \begin{align}
        L\omega + \omega L^T + QQ^T = 0.
    \end{align}
\end{theorem}

\subsection{Interventional SDEs}

We focus on the setting where we only observe the SDE through the induced stationary density under $k$ different shift interventions.   Specifically, there are vectors $\{c_i\}_{i=1}^k$ with each $c_i \in \mathbb{R}^n$, and we observe the stationary distribution of the SDE,
\begin{equation}\label{eq:sde}
    dX_t = (v(X_t) + c_i)dt + \sqrt{\epsilon}dB_t.
\end{equation}
% In the gene regulation setting, $c$ could represent CRISPR activation (one element of $c$ being positive) or interference (one element being negative) to enhance or repress gene expression, respectively.

% For example, in the gene regulation setting, these shifts could correspond to overexpression or knockdown in Perturb-seq data~\citep{dixit2016perturb}.  

We denote the concatenated intervention column vectors by the matrix $C \in \mathbb{R}^{n \times k}$.  In the causal disentanglement literature, shift interventions typically give fewer guarantees~\citep{buchholz2024learning,squires2023linear}, and even in the case of linear SDEs, the drift is trivially not identifiable without knowledge of the intervention vectors.  Therefore, we assume knowledge of the interventions $C$.
%%%%%%%%%%%%%%%%%%%%%%%%%%%%%%%%
\section{Related Work}

\subsection{Causal Representation}

In terms of modeling causality~\citep{pearl2009causality}, the most popular underlying model is the structural causal model (SCM), which characterizes the conditional distribution of a random variable under arbitrary intervention.  Learning an SCM typically requires very strong assumptions such as sparsity~\citep{scholkopf2021toward}, interventional data~\citep{lachapelle2022disentanglement}, parametric assumptions~\citep{peters2014identifiability}, among many other results.  Sparsity is a very common theme in these models, though it may also be expressed in an assumption that the number of latent variables is small, i.e. a low-rank constraint~\citep{fang2023low}.

\subsection{Causal Disentanglement with Interventional Data}

The bulk of the literature on causal disentanglement focuses on SCMs with an underlying DAG.  Relevant to our work are results that assume access to multiple interventional environments, either acting directly on the observed variables~\citep{brouillard2020differentiable} or identifying a latent model under some distributional assumptions~\citep{lachapelle2022disentanglement,squires2023linear,buchholz2024learning}.

Some of these works have addressed the crucial limitation of acyclicity~\citep{zheng2018dags, lee2019scaling,atanackovic2023dyngfn}, but without necessarily incorporating dynamics.  Many works require hard interventions, where an intervened variable is a function of exogenous noise, although some can handle soft interventions~\citep{zhang2024identifiability}.

\subsection{Dynamical System Methods}

Previous work has considered modeling perturbations with dynamical systems~\citep{peters2022causal}.  One can prove identifiability of SDE parameters from the generator~\citep{hansen2014causal} or trajectories~\citep{guan2024identifying, rajendran2024interventional}.  Other methods work in our harder setting where observed data is drawn from the stationary distribution under an SDE, and match a learned SDE under numerous interventions.  These works focus on linear drift and intervention-dependent parameters~\citep{rohbeck2024bicycle} or non-linear drift with shift interventions~\citep{lorch2024causal}.  Notably, neither paper gives theory to confirm if these models are identifiable.

There are also methods specific to a particular scientific domain.  In genomics, some methods act on pseudotime, an inferred notion of time from cell states when very few ``real'' timepoints are available~\citep{wang2024regvelo,hossain2024biologically}.  Although harder to obtain due the destructive nature of RNA sequencing, genuine temporal data (even with very few timepoints) can also be modeled with the intent of extracting a GRN~\citep{lin2025interpretable}.

The closest work to ours proves an identifiability bound for linear SDEs under a strong sparsity assumption~\citep{dettling2023identifiability}.  However, their result doesn't consider any interventional data, and the exact pattern of sparsity in the drift matrix is assumed to be known a priori, which is rarely the case in applied problems of interest.  Another closely related work is~\citet{guan2024identifying}, which can simultaneously infer the drift and diffusion.  However, this work applies only in an easier setting that assumes linearity, doesn't study interventions, and assumes multiple temporal marginals rather than just the stationary distribution.  Similarly~\citet{wang2024generator} proves identifiability with interventions but assumes access to the distribution over trajectories.

%%%%%%%%%%%%%%%%%%%%%%%%%%%%%%%%
\section{Main Results}

We consider two main parameterizations of the drift as linear or a two-layer neural network (an MLP) to verify when the model may be uniquely identified up to appropriate invariances.  Prior work has considered sparsity in the linear setting only~\citep{dettling2023identifiability}, but we focus on parameterizations that project to a low-dimensional space.  In other words, although the ambient dimension $n$ may be large, we assume the hidden dimension $r$ is much smaller, at least $n>2r$, and ideally the number of interventions to uniquely identify the parameters should scale with $r$.  This focus is driven by the empirical observation that high-dimensional dynamics are often driven by a low-dimensional subset, for example in gene modules~\citep{segal2005learning}.

\subsection{Linear case}

We start with the linear case and consider the parameterization $v(x) = (AB - D)x$ where $A \in \mathbb{R}^{n \times r}$, $B \in \mathbb{R}^{r \times n}$, $D \in \mathbb{R}^{n \times n}$.  Clearly $AB$ is a redundant parameterization of a rank $r$ matrix, but we use this notation to contrast the non-linear setting in Section~\ref{sec:non-linear}.  To ensure the drift is Hurwitz and the SDE has a stationary distribution, we assume $\|AB\| \leq \gamma < 1$ and $D \succeq I$.  Intuitively, $AB$ drive the dynamics while $D$ is a decay term to prevent unbounded dynamics.  This parameterization is similar to the one used in~\citet{rohbeck2024bicycle}, which also uses linear SDEs but considers hard interventions with having entirely new rows in the drift matrix, rather than shift interventions. 

% In our biological context, the low rank assumption is 1) intuitive in that genes are expected to behave in coordinated ways in modules/pathways, and 2) empirically proven, in that almost all successful single-cell methods leverage dimensionality reduction. 

It is impossible to get a good deterministic guarantee if the dynamics and interventions are chosen adversarially, as seen in the following proposition:

\begin{proposition}\label{prop:linear_deterministic}
    In the linear setting, there exist choices for parameters $A, B, D$ and interventions $C$ such that the drift matrix $AB - D$ is not identifiable with less than $n-r$ interventions.
\end{proposition}

The proof is provided in Appendix~\ref{sec:linear_deterministic}. Intuitively, one can choose interventions that don't affect the system dynamics at all.  But this situation is pathological, and in the linear case under some weak distributional assumptions we can show identifiability.

\begin{assumption}\label{ass:sampling}
    The matrices $A$ and $B$ are drawn from a density such that they are almost surely (a.s.) full-rank, have spectral norm less than some fixed $\gamma < 1$, and are invariant to applying a rotation matrix on the left or the right.
    Each column of $C$ is drawn iid from some distribution with a density on $\mathbb{R}^n$.  The decay matrix $D$ is observed.
\end{assumption}

We take care to explain why these assumptions are not particularly restrictive.  The low-rank constraint is already enforced by the drift function so $A$ and $B$ are almost surely full-rank when drawn from any density.  The spectral norm bound is necessary to guarantee the SDE has a stationary distribution.  And the rotational invariance assumption encodes an uninformative prior on which low-rank subspace governs the dynamics.  Two simple choices that satisfy these assumptions are sampling $A$ and $B^T$ from either the uniform measure on the Stiefel manifold (rectangular matrices with orthonormal columns) or with iid Gaussian entries, and then scaling down to ensure a spectral norm strictly smaller than one, with $C$ sampled from any density.

Additionally, for theoretical tractability, we presume knowledge of the decay term $D$.  This is a stronger assumption, but plausible in some applied contexts.  For example, in single-cell genomics the decay rate of genes can be estimated with external experiments (e.g., from BRIC-seq~\citep{imamachi2014bric}).

Altogether, we can now present a nearly tight identifiability result in the linear case.

\begin{theorem}\label{thm: linear_tight}
    Consider the linear drift in Equation~\ref{eq:sde}.    Then under Assumption~\ref{ass:sampling}, the drift $AB - D$ is identifiable a.s. with $r$ interventions, and unidentifiable a.s. with at most $r-2$ interventions.
\end{theorem}

See Appendix~\ref{sec:linear_tight} for the proof. Naively, one might count $2nr$ unknown parameters in the entries of $A$ and $B$, and assume the $n^2$ entries of the stationary covariance $\omega$ would be enough to identify them, but this isn't the case.  One needs exactly enough interventions to account for the hidden rank.

\paragraph{Identifying the decay.} If diagonal decay term $D$ is not inferred in advance, learning it simultaneously is comparable to the setting of robust PCA~\citep{candes2011robust} or recovery of a diagonal plus a positive semidefinite low rank term~\citep{saunderson2012diagonal}.  However, unlike the usual matrix completion setting where entries of the matrix are revealed uniformly at random, we have access to correlated low-rank measurements of the form $e_ic_j^T$ for $\{e_i\}_{i=1}^n$ the elementary basis.  This setting is well studied with sub-Gaussian measurement vectors~\citep{zhong2015efficient}, but the tools do not readily apply to our setting with non-random measurements and the constraints of the Lyapunov equation.  Nevertheless, we observe in Section~\ref{sec:experiments} that identifiability with an unknown diagonal decay matrix $D$ is empirically achieved, while still subject to the lower bound established in Theorem~\ref{thm: linear_tight}.

\subsection{Nonlinear case}\label{sec:non-linear}

The nonlinear case is more challenging, because there is no longer a nearly closed form characterization of the stationary distribution.  To apply tools from above, we consider when a linearized SDE can approximately capture the true dynamics, by restricting to contractive drift and small noise.  We consider the vector field $v(x) = A\sigma(Bx) - x$, with the constraints that $\|A\|, \|B\| \leq 1$.

\begin{assumption}\label{ass:weak_nonlinear}
    The matrices $A$ and $B^T$ are sampled uniformly among matrices with orthonormal columns i.e. the Stiefel manifold, each column of $C$ is drawn iid from an isotropic Gaussian. The map $\sigma \in \mathcal{C}^2$ is odd and acts on each element independently and satisfies:
    \begin{enumerate}
        \item $\max_{x\in\mathbb{R}} \sigma_i'(x) = \gamma < 1$
        \item $\min_{x\in\mathbb{R}} \sigma_i'(x) = \tau > 0$
        \item $\max_{x\in\mathbb{R}} |\sigma_i''(x)| = M < \infty$
        \item $Var_{g \sim \mathcal{N}(0,1)}(1/\sigma_i'(g)) \geq \nu > 0$ for all $i$.
        \item The set $\{x \in \mathbb{R}: \sigma_i''(x) = 0\}$ is measure zero.
    \end{enumerate}
\end{assumption}

The stronger assumptions on $A$ and $B^T$ are necessary for quantiative bounds.  Furthermore, their columns have unit norm, which is necessary to prevent scaling issues that would lead to non-identifiability where one could move constant factors from $A$ or $B$ into the definition of $\sigma$.

The strongest constraint here is condition 1, upper bounding the first derivative of the activation, as it implies the noiseless dynamics are globally contractive, but it guarantees a stationary distribution exists for any intervention vector $c$. Furthermore, these conditions still allow for a wide class of possible activations $\sigma$, mainly constrained to be increasing with bounded first and second derivative.  We observe the improved expressiveness of the stationary distribution of nonlinear SDEs in Figure~\ref{fig:contour}.  Note constraint 4 and 5 rule out functions that are linear (or locally linear).  This enables a stronger possible identifiability guarantee than only recovering $A$ and $B$ up to rotation as in the linear case.

\begin{figure}
     \centering
        \includegraphics[width=0.5\textwidth]{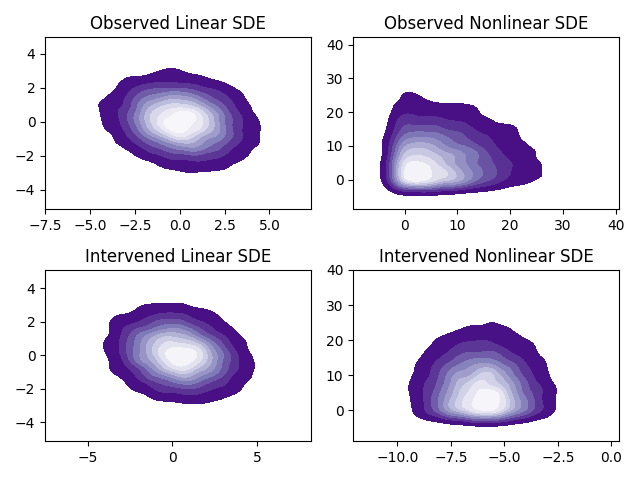}
        \caption{Contour plot of the stationary SDE under different activations and interventions.  Activation contractivity enforces one mode, but the linear distribution is Gaussian with fixed covariance across interventions, while the nonlinear distribution can be more expressive.}
        \label{fig:contour}
\end{figure}

\paragraph{Noiseless Setting.}

One might be tempted to simply consider the noiseless case, where the SDE reduces to an ordinary differential equation (ODE), and rather than recovering the stationary distribution one recovers instead the unique global stable point.  However, the noiseless case cannot take advantage of the low-rank structure, and requires $\Omega(n)$ interventions a.s. for identifiability.

\begin{proposition}\label{prop:ode_failure}
    Setting $\epsilon = 0$ in equation~\ref{eq:sde}, the parameters are a.s. not identifiable with fewer than $n-r$ interventions.
\end{proposition}

The proof is in Appendix~\ref{sec:ode_failure}. Intuitively, the equilibrium points of the noiseless ODE approximate the mean of the SDE for small noise (note this is not true for larger noise, for example see~\citet{ma2015complete}).  This suggests that in order to make use of the low-rank assumption on the dynamics, one must at least take advantage of second order moments of the SDE even in the small noise limit.

\paragraph{Moments in the zero-noise limit.}

As the noise converges to zero, the stationary distribution converges to a dirac centered on the global stable point, and there are no higher order moments.  However, the rescaled second-order moments have a non-trivial limit as noise goes to zero:

\begin{theorem}\label{thm:mean_cov}
   Let $x^*$ be unique solution $v(x^*) + c = 0$, and define $L = Jv(x^*)$.  If $\omega$ solves the Lyapunov equation $L\omega + \omega L^T + I = 0$, and $m_\epsilon$ and $\Sigma_\epsilon$ are the mean and covariance of the stationary distribution of Equation~\eqref{eq:sde} intervened by $c$, we have for sufficiently small $\epsilon$,
    \begin{align}
        \|m_\epsilon - x^*\| &\leq \left(\frac{\epsilon n}{1-\gamma}\right)^{1/2}, \\
        \|\Sigma_\epsilon/\epsilon - \omega\| &\lesssim \frac{\epsilon^{1/2} r^{5/2} n^{1/2} M}{(1-\gamma)^3}.
    \end{align}
\end{theorem}
The proof is given in Appendix~\ref{sec:mean_cov}. This is one instance of perturbation theory for SDEs~\citep{gardiner2021elements,sanz2017gaussian}.  Equipped with this fact, one can consider access to the stationary distribution of an SDE, specifically the first and second-order moments, and inspect what happens now as the noise goes to zero.  Due to issues with scaling and permutation of $A$ and $B$ being unidentifiable from the model, we measure recovery using the signed permutation distance $d_P$.

\begin{theorem}\label{thm:nonlinear_quant}
    Under Assumption~\ref{ass:weak_nonlinear} and prior knowledge of $\sigma$, suppose we observe the first and second moments $m_\epsilon$ and $\Sigma_\epsilon / \epsilon$ of the stationary distribution of the SDE in Equation~\eqref{eq:sde}.  If we have $k \gtrsim r^2$ interventions and sufficiently large $n$, then with probability at least $1 - 2\exp(-\tilde{C}k^{2/3}) - 2\exp(-\tilde{C}n^{1/3})$, $d_P(\tilde{A}, A) \leq \epsilon^{1/2}(poly(n,r,k)$ and $d_P(\tilde{B}^T, B^T) \leq \epsilon^{1/2}(poly(n,r,k)$.
\end{theorem}

These polynomial dependencies treat all other problem parameters, i.e. $\gamma, \tau, M, \nu$ as constants.  Without prior knowledge of the activation, we can still guarantee identifiability but with a limiting bound instead:

\begin{theorem}\label{thm:nonlinear_upper}
    Under Assumption~\ref{ass:weak_nonlinear}, suppose we observe the first and second moments $m_\epsilon$ and $\Sigma_\epsilon / \epsilon$ of the stationary distribution of the SDE in Equation~\eqref{eq:sde}.
 Then if we have $k \gtrsim r^2$ interventions and sufficiently large $n$, then for any fixed $t$, the probability that $d_P(A, \tilde{A}) > t$ or $d_P(B, \tilde{B}) > t$ goes to $0$ as $\epsilon \rightarrow 0$.
\end{theorem}

We give a brief sketch of the proof, to capture the novel elements and how the identifiability of $A$ and $B$ appears here but not in the linear case.  The zero-noise limit approximately linearizes each intervened SDE.  Unlike the exactly linear case, the stationary covariance is not identical across interventions.  By Theorem~\ref{thm:mean_cov}, the covariance under the $i$th intervention approximately matches the SDE with drift $A(J\sigma(Bx_i^*))B$ where $x_i^*$ is the unique zero of the intervened drift $v(x) + c_i$.

Sufficient variability in these fixed points (guaranteed by our assumptions on $\sigma$) enables the identification of $A$ and $B$.  Suppose one had access to the drift matrices: because the Jacobian term in the drift is diagonal, a linear combination of drifts across $r$ interventions can yield a rank-one matrix, which must correspond to the outer product of a column of $A$ and a row of $B$.  But we don't have access to the drift matrices directly, only the covariances they induce, and so the proof relies on finding a different set of matrices.  The robust version of this argument is based on the tensor decomposition analysis in~\citet{bhaskara2014uniqueness}.

We hypothesize the constraint that the SDE drift have a single absorbing point is unnecessary.  If the ODE given by $\frac{dx}{dt} = v(x) + c_i$ had multiple stable equilibria, then in the limit of small noise, the SDE stationary distribution should approach a Gaussian mixture centered around each equilibrium point.  The proof only requires certain linear independence conditions of vectors induced by each fixed point, so one intervention with multiple equilibria would offer essentially the same information as additional unimodal interventions.  This is consistent with the observations in~\citet{lorch2024causal} that nonlinear SDE models generalize well with a limited number of interventions.

%%%%%%%%%%%%%%%%%%%%%%%%%%%%%%%%
\section{Experiments}\label{sec:experiments}

To validate the theory given above, we consider experiments demonstrating the recovery of SDE parameters and generalization to unseen interventions.  To show the broad applicability of the theory, we consider multiple loss functions: a loss acting directly on the parameters in the linear case~\citep{rohbeck2024bicycle}, a kernelized Stein discrepancy~\citep{barp2019minimum}, and rollout loss for neural SDEs~\citep{kidger2021neural}.

\subsection{Linear SDE Recovery}

Because we are fitting linear SDEs where the stationary distribution is Gaussian, we can choose a very simple loss that matches the population mean and covariance of each intervented distribution, to verify the claim of identifiability (Theorem~\ref{thm: linear_tight}).  This loss is akin to the one used in the Bicycle method introduced in~\citet{rohbeck2024bicycle} that fits interventional linear SDEs.  Assuming the noise scale $\epsilon$ is known and $L$ denotes the true drift, we fit the parameters $\hat{A}$, $\hat{B}$, and $\hat{D}$.  Letting our estimate for the drift be denoted by $\hat{L} := \hat{A}\hat{B} - \hat{D}$, the loss function matches the mean and covariance of each intervention:
\begin{equation}\label{eq:lin_loss}
    \mathcal{L}_{lin}(\hat{A}, \hat{B}, \hat{D}) = \|\hat{L}^{-1}C - L^{-1}C\|_F + \|\hat{L} \omega + \omega \hat{L}^T + \epsilon I\|_F.
\end{equation}

We evaluate the recovery of $AB$ under different numbers of interventions for two different ambient dimensions $n$ and two different true ranks $r$ in Figure~\ref{fig:linearsde}.  For each sampled SDE, we use the best train error from 100 independent initializations to deal with the nonconvexity of the objective.  The only exception is the oversampled $k=r\log(n)$ where we need only 5 initializations, as the landscape is empirically easier to learn.  Further details are given in Section~\ref{sec:experimental_details_appendix}.

We confirm our theory: when the decay is known, $r-2$ interventions fails to grant identifiability with poor performance and extremely high variance, even with many independent initializations, while $r$ interventions clearly succeed. When the decay isn't known, we use a modest oversampling of $r\log(n)$ interventions.  This amount is informed by the literature on matrix completion, where $\Omega(rn\log(n))$ revealed matrix entries were proven to be necessary for a rank $r$ matrix recovery problem~\citep{candes2010power}. We scale down by a factor of $n$ because in our setting each intervention yields the mean vector of the perturbed linear SDE in $\mathbb{R}^n$, hence $n$ measurements.

\begin{figure*}[h]
     \centering
     \begin{subfigure}[b]{0.3\textwidth}
         \centering
         \includegraphics[width=\textwidth]{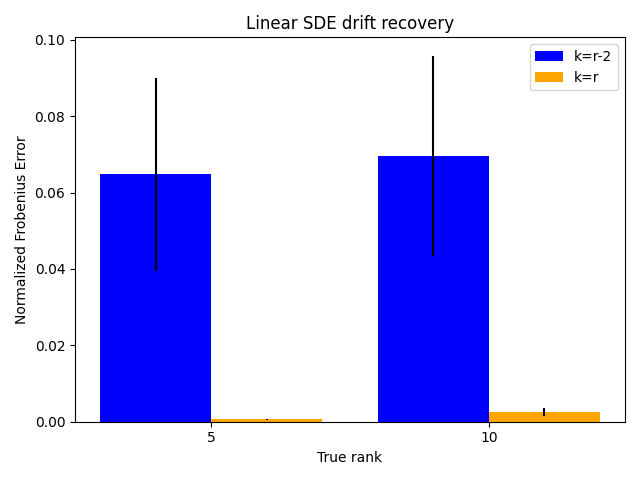}
         \caption{$n=100$ with fixed decay}
         \label{fig:linear100}
     \end{subfigure}
     \hspace{0.05\textwidth}
     \begin{subfigure}[b]{0.3\textwidth}
         \centering
         \includegraphics[width=\textwidth]{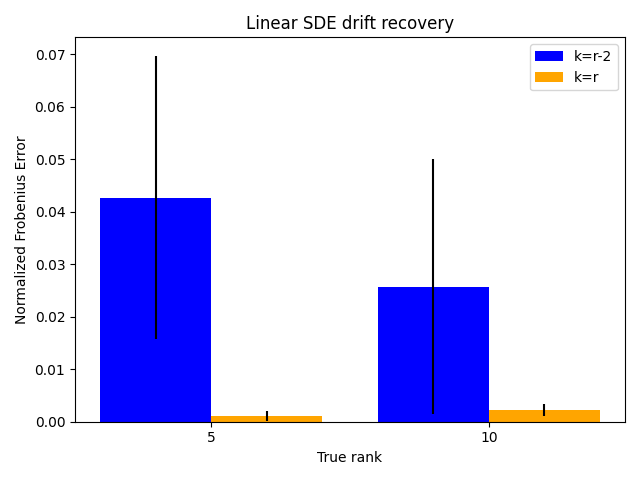}
         \caption{$n=200$ with fixed decay}
         \label{fig:linear200}
     \end{subfigure}
     
     \begin{subfigure}[b]{0.3\textwidth}
         \centering
         \includegraphics[width=\textwidth]{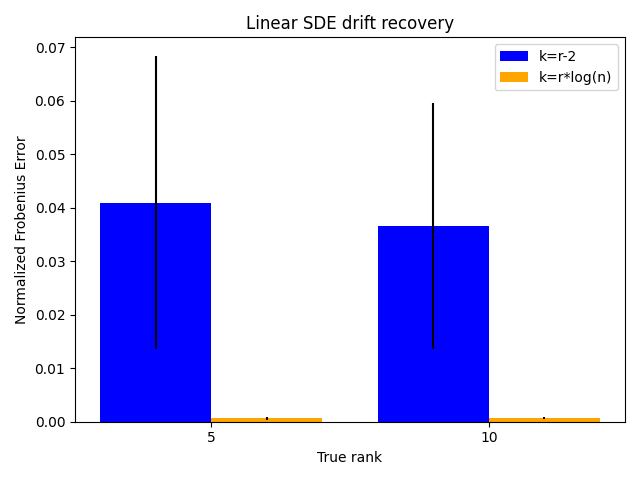}
         \caption{$n=100$ with learned decay}
         \label{fig:linear100_False}
     \end{subfigure}
     \hspace{0.05\textwidth}
     \begin{subfigure}[b]{0.3\textwidth}
         \centering
         \includegraphics[width=\textwidth]{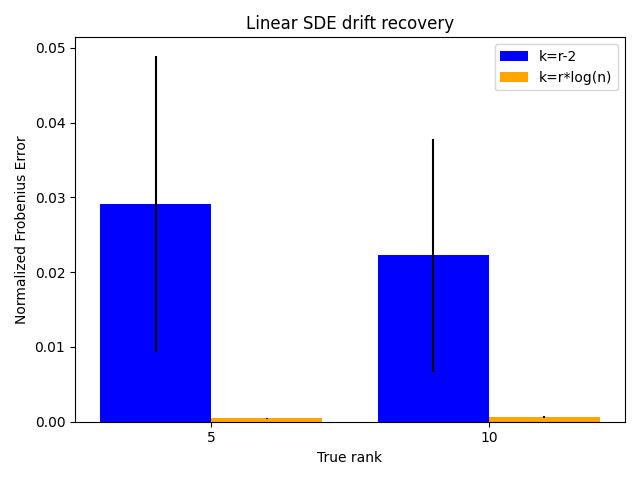}
         \caption{$n=200$ with learned decay}
         \label{fig:linear200_False}
     \end{subfigure}
        \caption{Normalized Frobenius error of learned drift against true drift in linear SDEs with $k$ independent Gaussian interventions.  Error bars are standard deviation over 5 independent runs.}
        \label{fig:linearsde}
\end{figure*}

\subsection{Nonlinear SDE Recovery}\label{sec:non_linear_sde_recovery}

We repeat this verification of identifiability in the setting of Theorem~\ref{thm:nonlinear_quant}, using a very small value of $\epsilon=10^{-5}$ to approximately linearize.  For some randomly chosen $A$ and $B$, we define the true drift $v(x) = A\sigma(Bx) - x$ and use each intervened drift $v + c_i$ to calculate the corresponding sample mean $m_\epsilon^i$ and sample covariance $\Sigma_\epsilon^i$.  

We define a parameterized drift of the form $\hat{v}(x) = \hat{A}\sigma(\hat{B}x) - x$ where $\sigma(x) = 0.7x + \tanh(x)$ as this satisfies the Assumption~\ref{ass:weak_nonlinear}.  We train with the loss function

\begin{align*}
    \mathcal{L}_{nonlin}(\hat{A}, \hat{B}) = \sqrt{\sum_{i=1}^k \left\|\hat{v}(m_\epsilon^i) + c_i\right\|^2} \\ + \sqrt{\sum_{i=1}^k \left\|J\hat{v}(m_\epsilon^i)(\Sigma_\epsilon^i / \epsilon) + (\Sigma_\epsilon^i / \epsilon) J\hat{v}(m_\epsilon^i)^T + I\right\|^2}
\end{align*}

The motivation for this loss comes from Theorem~\ref{thm:mean_cov}.  The first term enforces the approximate fixed point property of the mean, and the second term enforces the approximate Lyapunov equation of the covariance.
Thus we can train $\hat{v}$ to match the first and second moments of the true stationary distribution of $v$, which will be approximately Gaussian for sufficiently small $\epsilon$.

To empirically evaluate identifiability, we calculate the normalized error in recovering $A$ and $B$ from the parameters $\hat{A}$ and $\hat{B}$, chosen from the run with minimum training loss across 100 independent runs.  Further details are given in the Appendix.  We observe that $r^2$ interventions enable consistently good recovery of the underlying drift parameters in the known activation setting.

\begin{table}[h]
 \caption{Normalized Frobenius error of drift parameters with $k$ independent Gaussian interventions, over 10 independent runs.}
 \label{tab:non_linear_results}
\centering
 \begin{tabular}{|c|c|c|} 
 \hline
 & $A$ Error & $B$ Error \\
  \hline
 $k=r^2$ &$0.03 \pm 0.02$ & $0.03 \pm 0.01$ \\
 \hline
 \end{tabular}
\end{table}

\subsection{Synthetic Nonlinear SDE Generalization}

For an evaluation of Theorem~\ref{thm:nonlinear_upper}, we apply the insight that identifiability doesn't require knowing $\sigma$ and consider learnable MLPs.  Learnable activations~\citep{goyal2019learning} have been proposed before~\citep{apicella2019simple} but to our knowledge not previously applied to SDE parameterizations.

To assess generalizability, we consider a loss function proposed by~\citet{lorch2024causal}: the kernel deviation from stationarity (KDS), or equivalently the kernelized Stein discrepancy of the stationary distribution~\citep{barp2019minimum}.  When the parametric drift is defined as $v_\theta(x) = A\sigma(Bx) - Dx$ and the target stationary distribution is $\mu$, the KDS is,
\begin{equation}\label{eq:kds_loss}
    \mathcal{L}_{KDS}(\theta) = E_{x\sim\mu}E_{x'\sim\mu} \mathcal{A}_x \mathcal{A}_{x'} k(x,x').
\end{equation}
where $k$ is a given kernel and $\mathcal{A}_x$ is the SDE generator acting on the variable $x$ (for more details see~\citet{lorch2024causal}).  We use samples from the true stationary distribution $\mu$ to approximate these expectations, further training details are given in the Appendix.  This loss is nevertheless unsuitable when using noise small enough to reach the limit proven in Theorem~\ref{thm:nonlinear_upper}, and therefore we assess generalizability according to performance on unseen interventions.

We follow~\citet{zhang2023active,lorch2024causal} and evaluate with the mean squared error (MSE) between the true and predicted distribution.  Further details are given in Section~\ref{app:exp_det_nonlin}.  The results are given in Table~\ref{tab:kds_results}. We observe a clear benefit for low noise SDEs, which gets smaller as the noise gets bigger and further away from our proven settings.  This is consistent with our theory, and also intuitive: as the noise gets larger, the stationary distribution gets smoother and has less dependence on the intervention vectors.  Nevertheless, in the low noise regime we see that learnable activations improve prediction on test interventions without overfitting.

\begin{table}[h]
 \caption{Mean distance between true and predicted distribution with $n=20$, $r=3$, over 20 independent seeds with 20 test interventions per seed, with sigmoid activation versus learned activation.}
 \label{tab:kds_results}
\centering
 \begin{tabular}{|c|c|c|} 
 \hline
 & Sigmoid & Activation MLP \\
  \hline
 $\epsilon=0.05$ & $21.78 \pm 12.78$ & $\bm{9.51 \pm 1.28}$\\
 \hline
 $\epsilon=0.1$ & $16.54 \pm 6.35$ & $\bm{9.23 \pm 1.65}$\\
 \hline
 $\epsilon=0.2$ & $11.53 \pm 3.47$ & $\bm{7.96 \pm 1.44}$\\
 \hline
 $\epsilon=0.3$ & $7.97 \pm 2.19$ & $\bm{6.69 \pm 2.07}$\\
 \hline
 \end{tabular}
\end{table}

\subsection{Simulated GRN SDE Generalization}\label{sec:grn-recovery}

\begin{figure*}[ht!]
     \centering
     \includegraphics[width=0.5\textwidth]{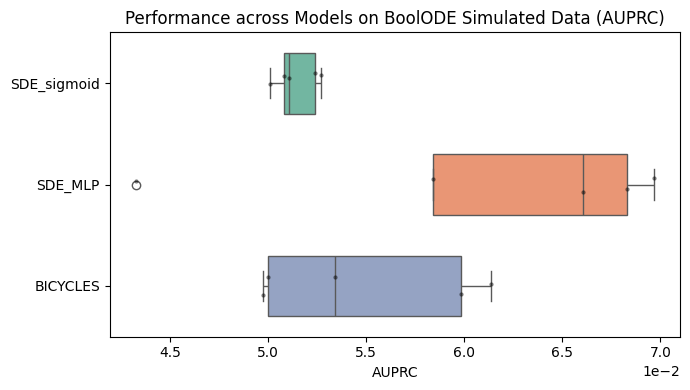}
    \caption{Gene regulatory network recovery on three tested SDE models for 5 independent runs.}
    \label{fig:grn_recovery}
\end{figure*}

We consider an applied setting for learnable activations, and how they impact the recovery of GRNs.  In the genomics context, $n$ is the number of expressed genes in the thousands, and $r$ corresponds to a much smaller number of latent gene modules/pathways.  Because there are few fully known regulatory networks we rely on semi-synthetic data to produce data with similar characteristics to true perturbed transcriptomic samples and a ground-truth GRN.

We use the PerturbODE model from~\citet{lin2025interpretable}, similar to our setup and modified to handle a loss for parameterized SDEs~\citep{kidger2021neural}.  We consider the inferred GRN as the matrix multiplication of the weight matrices in the MLP that parameterizes the drift, corresponding to a first-order Taylor approximation of the dynamics.  GRN recovery can be measured as a classification task of individual edges.  As a baseline, we also consider the Bicycle interventional linear SDE method~\citep{rohbeck2024bicycle}.  To simulate a known GRN, we use the boolean ODE based-simulator BEELINE~\citep{pratapa2020benchmarking}. BEELINE uses the provided GRN to parametrize an SDE in the space of mRNA expression and protein expression, with physically plausible regulation dynamics.  More precise details are given in the Appendix in Section~\ref{app:exp_det_beeline}. 

GRN recovery is still extremely difficult (random guessing is $\approx 0.0586$ and maximum AUPRC for dynamical methods $\approx 0.07$), but we observe improvement using nonlinear SDEs over linear SDEs, and substantial improvement using learnable activation in the SDE parameterization (Figure~\ref{fig:grn_recovery}). 

%%%%%%%%%%%%%%%%%%%%%%%%%%%%%%%%
\section{Discussion}\label{sec:discussion}

We confirm identifiability directly in the linear case and non-linear, low-noise limit, and see improved generalization using learned activations for small noise.  It may be beneficial to increase the capacity of the SDE beyond two-layer neural networks, if the extracted drift were still identifiable.

We note some limitations of the given theory.  The results demonstrate identifiability in a setup with infinitesimal small noise and restrictions on the possible drift functions.  Generalizing to larger noise regimes may be possible with higher order moment information about the stationary distribution, although understanding the stationary distribution outside of linearization is very challenging.  Broader parameterizations of the drift or interventions are interesting extensions for future work.
%%%%%%%%%%%%%%%%%%%%%%%%%%%%%%%%
\section{Conclusion}

In this work, we've given the first provable results regarding identifiability of interventional SDEs, in both the linear and nonlinear case.  Although such models are currently less popular and less studied then comparable SCM-based modeling, the ability to obtain provable guarantees for dynamical systems without any temporal data suggests they may become more fruitful for causal inference in the future.  As longitudinal genomics datasets often have a small number of time-points, future work may consider how to incorporate sparse trajectory information with the stationary distribution for better identifiability guarantees or more effective architectures for recovering regulatory networks.

\begin{acknowledgements} % will be removed in pdf for initial submission,
This work was made possible by support from the MacMillan Family and the MacMillan Center for the Study of the Non-Coding Cancer Genome at the New York Genome Center. A.Z. is the Sijbrandij Foundation Quantitative Biology Fellow of the Damon Runyon Cancer Research Foundation (DRQ26-25). This work was also supported by the NIH NHGRI grant R01HG012875, and grant number 2022-253560 from the Chan Zuckerberg Initiative DAF, an advised fund of Silicon Valley Community Foundation.
\end{acknowledgements}

% References
\bibliography{uai2026-template}

@inproceedings{rohbeck2024bicycle,
  title={Bicycle: Intervention-Based Causal Discovery with Cycles},
  author={Rohbeck, Martin and Clarke, Brian and Mikulik, Katharina and Pettet, Alexandra and Stegle, Oliver and Ueltzh{\"o}ffer, Kai},
  booktitle={Causal Learning and Reasoning},
  pages={209--242},
  year={2024},
  organization={PMLR}
}

@article{tropp2012user,
  title={User-friendly tail bounds for sums of random matrices},
  author={Tropp, Joel A},
  journal={Foundations of computational mathematics},
  volume={12},
  number={4},
  pages={389--434},
  year={2012},
  publisher={Springer}
}

@article{sachs2005causal,
  title={Causal protein-signaling networks derived from multiparameter single-cell data},
  author={Sachs, Karen and Perez, Omar and Pe'er, Dana and Lauffenburger, Douglas A and Nolan, Garry P},
  journal={Science},
  volume={308},
  number={5721},
  pages={523--529},
  year={2005},
  publisher={American Association for the Advancement of Science}
}

@article{segal2005learning,
  title={Learning module networks.},
  author={Segal, Eran and Pe'er, Dana and Regev, Aviv and Koller, Daphne and Friedman, Nir and Jaakkola, Tommi},
  journal={Journal of Machine Learning Research},
  volume={6},
  number={4},
  year={2005}
}

@inproceedings{varando2020graphical,
  title={Graphical continuous Lyapunov models},
  author={Varando, Gherardo and Hansen, Niels Richard},
  booktitle={Conference on Uncertainty in Artificial Intelligence},
  pages={989--998},
  year={2020},
  organization={Pmlr}
}

@article{hening2021stationary,
  title={Stationary distributions of persistent ecological systems},
  author={Hening, Alexandru and Li, Yao},
  journal={Journal of Mathematical Biology},
  volume={82},
  number={7},
  pages={64},
  year={2021},
  publisher={Springer}
}

@inproceedings{feydy2019interpolating,
  title={Interpolating between optimal transport and mmd using sinkhorn divergences},
  author={Feydy, Jean and S{\'e}journ{\'e}, Thibault and Vialard, Fran{\c{c}}ois-Xavier and Amari, Shun-ichi and Trouv{\'e}, Alain and Peyr{\'e}, Gabriel},
  booktitle={The 22nd international conference on artificial intelligence and statistics},
  pages={2681--2690},
  year={2019},
  organization={PMLR}
}

@book{evans2025measure,
  title={Measure theory and fine properties of functions},
  author={Evans, Lawrence C},
  year={2025},
  publisher={Chapman and Hall/CRC}
}

@article{zheng2018dags,
  title={Dags with no tears: Continuous optimization for structure learning},
  author={Zheng, Xun and Aragam, Bryon and Ravikumar, Pradeep K and Xing, Eric P},
  journal={Advances in neural information processing systems},
  volume={31},
  year={2018}
}

@article{brouillard2020differentiable,
  title={Differentiable causal discovery from interventional data},
  author={Brouillard, Philippe and Lachapelle, S{\'e}bastien and Lacoste, Alexandre and Lacoste-Julien, Simon and Drouin, Alexandre},
  journal={Advances in Neural Information Processing Systems},
  volume={33},
  pages={21865--21877},
  year={2020}
}

@inproceedings{lachapelle2022disentanglement,
  title={Disentanglement via mechanism sparsity regularization: A new principle for nonlinear ICA},
  author={Lachapelle, S{\'e}bastien and Rodriguez, Pau and Sharma, Yash and Everett, Katie E and Le Priol, R{\'e}mi and Lacoste, Alexandre and Lacoste-Julien, Simon},
  booktitle={Conference on Causal Learning and Reasoning},
  pages={428--484},
  year={2022},
  organization={PMLR}
}

@article{buchholz2024learning,
  title={Learning linear causal representations from interventions under general nonlinear mixing},
  author={Buchholz, Simon and Rajendran, Goutham and Rosenfeld, Elan and Aragam, Bryon and Sch{\"o}lkopf, Bernhard and Ravikumar, Pradeep},
  journal={Advances in Neural Information Processing Systems},
  volume={36},
  year={2024}
}

@inproceedings{rajendran2024interventional,
  title={An interventional perspective on identifiability in gaussian lti systems with independent component analysis},
  author={Rajendran, Goutham and Reizinger, Patrik and Brendel, Wieland and Ravikumar, Pradeep Kumar},
  booktitle={Causal Learning and Reasoning},
  pages={41--70},
  year={2024},
  organization={PMLR}
}

@inproceedings{squires2023linear,
  title={Linear causal disentanglement via interventions},
  author={Squires, Chandler and Seigal, Anna and Bhate, Salil S and Uhler, Caroline},
  booktitle={International Conference on Machine Learning},
  pages={32540--32560},
  year={2023},
  organization={PMLR}
}

@article{scholkopf2021toward,
  title={Toward causal representation learning},
  author={Sch{\"o}lkopf, Bernhard and Locatello, Francesco and Bauer, Stefan and Ke, Nan Rosemary and Kalchbrenner, Nal and Goyal, Anirudh and Bengio, Yoshua},
  journal={Proceedings of the IEEE},
  volume={109},
  number={5},
  pages={612--634},
  year={2021},
  publisher={IEEE}
}

@inproceedings{lorch2024causal,
  title={Causal Modeling with Stationary Diffusions},
  author={Lorch, Lars and Krause, Andreas and Sch{\"o}lkopf, Bernhard},
  booktitle={International Conference on Artificial Intelligence and Statistics},
  pages={1927--1935},
  year={2024},
  organization={PMLR}
}

@article{hossain2024biologically,
  title={Biologically informed NeuralODEs for genome-wide regulatory dynamics},
  author={Hossain, Intekhab and Fanfani, Viola and Fischer, Jonas and Quackenbush, John and Burkholz, Rebekka},
  journal={Genome Biology},
  volume={25},
  number={1},
  pages={127},
  year={2024},
  publisher={Springer}
}

@article{hansen2014causal,
  title={Causal interpretation of stochastic differential equations},
  author={Hansen, Niels and Sokol, Alexander},
  year={2014}
}

@article{wang2024generator,
  title={Generator identification for linear SDEs with additive and multiplicative noise},
  author={Wang, Yuanyuan and Geng, Xi and Huang, Wei and Huang, Biwei and Gong, Mingming},
  journal={Advances in Neural Information Processing Systems},
  volume={36},
  year={2024}
}

@article{dettling2023identifiability,
  title={Identifiability in continuous Lyapunov models},
  author={Dettling, Philipp and Homs, Roser and Am{\'e}ndola, Carlos and Drton, Mathias and Hansen, Niels Richard},
  journal={SIAM Journal on Matrix Analysis and Applications},
  volume={44},
  number={4},
  pages={1799--1821},
  year={2023},
  publisher={SIAM}
}

@misc{vershynin2025high,
  title={High-dimensional probability},
  author={Vershynin, Roman},
  year={2025},
  publisher={Cambridge University Press}
}

@book{meckes2019random,
  title={The random matrix theory of the classical compact groups},
  author={Meckes, Elizabeth S},
  volume={218},
  year={2019},
  publisher={Cambridge University Press}
}

@inproceedings{bhaskara2014uniqueness,
  title={Uniqueness of tensor decompositions with applications to polynomial identifiability},
  author={Bhaskara, Aditya and Charikar, Moses and Vijayaraghavan, Aravindan},
  booktitle={Conference on Learning Theory},
  pages={742--778},
  year={2014},
  organization={PMLR}
}

@incollection{peters2022causal,
  title={Causal models for dynamical systems},
  author={Peters, Jonas and Bauer, Stefan and Pfister, Niklas},
  booktitle={Probabilistic and Causal Inference: The Works of Judea Pearl},
  pages={671--690},
  year={2022}
}

@book{sarkka2019applied,
  title={Applied stochastic differential equations},
  author={S{\"a}rkk{\"a}, Simo and Solin, Arno},
  volume={10},
  year={2019},
  publisher={Cambridge University Press}
}

@inproceedings{lee2019scaling,
  title={Scaling structural learning with NO-BEARS to infer causal transcriptome networks},
  author={Lee, Hao-Chih and Danieletto, Matteo and Miotto, Riccardo and Cherng, Sarah T and Dudley, Joel T},
  booktitle={Pacific Symposium on Biocomputing 2020},
  pages={391--402},
  year={2019},
  organization={World Scientific}
}

@article{pratapa2020benchmarking,
  title={Benchmarking algorithms for gene regulatory network inference from single-cell transcriptomic data},
  author={Pratapa, Aditya and Jalihal, Amogh P and Law, Jeffrey N and Bharadwaj, Aditya and Murali, TM},
  journal={Nature methods},
  volume={17},
  number={2},
  pages={147--154},
  year={2020},
  publisher={Nature Publishing Group US New York}
}

@article{dibaeinia2020sergio,
  title={SERGIO: a single-cell expression simulator guided by gene regulatory networks},
  author={Dibaeinia, Payam and Sinha, Saurabh},
  journal={Cell systems},
  volume={11},
  number={3},
  pages={252--271},
  year={2020},
  publisher={Elsevier}
}

@article{candes2011robust,
  title={Robust principal component analysis?},
  author={Cand{\`e}s, Emmanuel J and Li, Xiaodong and Ma, Yi and Wright, John},
  journal={Journal of the ACM (JACM)},
  volume={58},
  number={3},
  pages={1--37},
  year={2011},
  publisher={ACM New York, NY, USA}
}

@article{saunderson2012diagonal,
  title={Diagonal and low-rank matrix decompositions, correlation matrices, and ellipsoid fitting},
  author={Saunderson, James and Chandrasekaran, Venkat and Parrilo, Pablo A and Willsky, Alan S},
  journal={SIAM Journal on Matrix Analysis and Applications},
  volume={33},
  number={4},
  pages={1395--1416},
  year={2012},
  publisher={SIAM}
}

@inproceedings{zhong2015efficient,
  title={Efficient matrix sensing using rank-1 gaussian measurements},
  author={Zhong, Kai and Jain, Prateek and Dhillon, Inderjit S},
  booktitle={Algorithmic Learning Theory: 26th International Conference, ALT 2015, Banff, AB, Canada, October 4-6, 2015, Proceedings 26},
  pages={3--18},
  year={2015},
  organization={Springer}
}

@book{waddington2014strategy,
  title={The strategy of the genes},
  author={Waddington, Conrad Hal},
  year={2014},
  publisher={Routledge}
}

@article{guan2024identifying,
  title={Identifying Drift, Diffusion, and Causal Structure from Temporal Snapshots},
  author={Guan, Vincent and Janssen, Joseph and Rahmani, Hossein and Warren, Andrew and Zhang, Stephen and Robeva, Elina and Schiebinger, Geoffrey},
  journal={arXiv preprint arXiv:2410.22729},
  year={2024}
}

@article{berglund2021long,
  title={Long-time dynamics of stochastic differential equations},
  author={Berglund, Nils},
  journal={arXiv preprint arXiv:2106.12998},
  year={2021}
}

@book{pearl2009causality,
  title={Causality},
  author={Pearl, Judea},
  year={2009},
  publisher={Cambridge university press}
}

@article{peters2014identifiability,
  title={Identifiability of Gaussian structural equation models with equal error variances},
  author={Peters, Jonas and B{\"u}hlmann, Peter},
  journal={Biometrika},
  volume={101},
  number={1},
  pages={219--228},
  year={2014},
  publisher={Oxford University Press}
}

@article{zhang2024identifiability,
  title={Identifiability guarantees for causal disentanglement from soft interventions},
  author={Zhang, Jiaqi and Greenewald, Kristjan and Squires, Chandler and Srivastava, Akash and Shanmugam, Karthikeyan and Uhler, Caroline},
  journal={Advances in Neural Information Processing Systems},
  volume={36},
  year={2024}
}

@article{wang2024regvelo,
  title={RegVelo: gene-regulatory-informed dynamics of single cells},
  author={Wang, Weixu and Hu, Zhiyuan and Weiler, Philipp and Mayes, Sarah and Lange, Marius and Wang, Jingye and Xue, Zhengyuan and Sauka-Spengler, Tatjana and Theis, Fabian J},
  journal={bioRxiv},
  pages={2024--12},
  year={2024},
  publisher={Cold Spring Harbor Laboratory}
}

@article{lin2025interpretable,
  title={Interpretable Neural ODEs for Gene Regulatory Network Discovery under Perturbations},
  author={Lin, Zaikang and Chang, Sei and Zweig, Aaron and Azizi, Elham and Knowles, David A},
  journal={arXiv preprint arXiv:2501.02409},
  year={2025}
}

@article{ma2015complete,
  title={A complete recipe for stochastic gradient MCMC},
  author={Ma, Yi-An and Chen, Tianqi and Fox, Emily},
  journal={Advances in neural information processing systems},
  volume={28},
  year={2015}
}

@article{dixit2016perturb,
  title={Perturb-Seq: dissecting molecular circuits with scalable single-cell RNA profiling of pooled genetic screens},
  author={Dixit, Atray and Parnas, Oren and Li, Biyu and Chen, Jenny and Fulco, Charles P and Jerby-Arnon, Livnat and Marjanovic, Nemanja D and Dionne, Danielle and Burks, Tyler and Raychowdhury, Raktima and others},
  journal={cell},
  volume={167},
  number={7},
  pages={1853--1866},
  year={2016},
  publisher={Elsevier}
}

@article{barp2019minimum,
  title={Minimum stein discrepancy estimators},
  author={Barp, Alessandro and Briol, Francois-Xavier and Duncan, Andrew and Girolami, Mark and Mackey, Lester},
  journal={Advances in Neural Information Processing Systems},
  volume={32},
  year={2019}
}

@inproceedings{kidger2021neural,
  title={Neural sdes as infinite-dimensional gans},
  author={Kidger, Patrick and Foster, James and Li, Xuechen and Lyons, Terry J},
  booktitle={International conference on machine learning},
  pages={5453--5463},
  year={2021},
  organization={PMLR}
}

@article{candes2010power,
  title={The power of convex relaxation: Near-optimal matrix completion},
  author={Cand{\`e}s, Emmanuel J and Tao, Terence},
  journal={IEEE transactions on information theory},
  volume={56},
  number={5},
  pages={2053--2080},
  year={2010},
  publisher={IEEE}
}

@article{goyal2019learning,
  title={Learning activation functions: A new paradigm for understanding neural networks},
  author={Goyal, Mohit and Goyal, Rajan and Lall, Brejesh},
  journal={arXiv preprint arXiv:1906.09529},
  year={2019}
}

@book{gardiner2021elements,
  title={Elements of Stochastic Methods},
  author={Gardiner, Crispin W},
  year={2021},
  publisher={AIP Publishing Melville, NY, USA}
}

@article{sanz2017gaussian,
  title={Gaussian approximations of small noise diffusions in Kullback--Leibler divergence},
  author={Sanz-Alonso, Daniel and Stuart, Andrew M},
  journal={Communications in Mathematical Sciences},
  volume={15},
  number={7},
  pages={2087--2097},
  year={2017},
  publisher={International Press of Boston}
}

@article{apicella2019simple,
  title={A simple and efficient architecture for trainable activation functions},
  author={Apicella, Andrea and Isgro, Francesco and Prevete, Roberto},
  journal={Neurocomputing},
  volume={370},
  pages={1--15},
  year={2019},
  publisher={Elsevier}
}

@article{atanackovic2023dyngfn,
  title={Dyngfn: Towards bayesian inference of gene regulatory networks with gflownets},
  author={Atanackovic, Lazar and Tong, Alexander and Wang, Bo and Lee, Leo J and Bengio, Yoshua and Hartford, Jason S},
  journal={Advances in Neural Information Processing Systems},
  volume={36},
  pages={74410--74428},
  year={2023}
}

@article{song2021maximum,
  title={Maximum likelihood training of score-based diffusion models},
  author={Song, Yang and Durkan, Conor and Murray, Iain and Ermon, Stefano},
  journal={Advances in neural information processing systems},
  volume={34},
  pages={1415--1428},
  year={2021}
}

@article{kingma2014adam,
  title={Adam: A method for stochastic optimization},
  author={Kingma, Diederik P},
  journal={arXiv preprint arXiv:1412.6980},
  year={2014}
}

@article{zhang2023active,
  title={Active learning for optimal intervention design in causal models},
  author={Zhang, Jiaqi and Cammarata, Louis and Squires, Chandler and Sapsis, Themistoklis P and Uhler, Caroline},
  journal={Nature Machine Intelligence},
  volume={5},
  number={10},
  pages={1066--1075},
  year={2023},
  publisher={Nature Publishing Group UK London}
}

@inproceedings{loshchilovdecoupled,
  title={Decoupled Weight Decay Regularization},
  author={Loshchilov, Ilya and Hutter, Frank},
  booktitle={International Conference on Learning Representations}
}

@article{fang2023low,
  title={On low-rank directed acyclic graphs and causal structure learning},
  author={Fang, Zhuangyan and Zhu, Shengyu and Zhang, Jiji and Liu, Yue and Chen, Zhitang and He, Yangbo},
  journal={IEEE Transactions on Neural Networks and Learning Systems},
  volume={35},
  number={4},
  pages={4924--4937},
  year={2023},
  publisher={IEEE}
}

@article{imamachi2014bric,
  title={BRIC-seq: a genome-wide approach for determining RNA stability in mammalian cells},
  author={Imamachi, Naoto and Tani, Hidenori and Mizutani, Rena and Imamura, Katsutoshi and Irie, Takuma and Suzuki, Yutaka and Akimitsu, Nobuyoshi},
  journal={Methods},
  volume={67},
  number={1},
  pages={55--63},
  year={2014},
  publisher={Elsevier}
}

\newpage

\onecolumn

\title{Towards Identifiability of Interventional Stochastic Differential Equations\\(Supplementary Material)}

\maketitle

\section{Proofs of Results}\label{sec:proof_of_results_appendix}
\subsection*{Proof of Proposition~\ref{prop:linear_deterministic}} \label{sec:linear_deterministic}
Consider block matrices such that,
    \begin{align}
    A &= \left[\begin{array}{ c | c }
    \frac{1}{\sqrt{2}} I_r & 0 \\
    \hline
    0 & 0
  \end{array}\right] \\
  B &= \left[\begin{array}{ c | c }
    \frac{1}{\sqrt{2}} I_r & 0 \\
    \hline
    0 & 0
  \end{array}\right] 
\end{align}
and set $D = I$ and $\epsilon = 1$. If we also set
\begin{align}
        \omega &= \left[\begin{array}{ c | c }
    I_r & 0 \\
    \hline
    0 & \frac{1}{2} I_{n-r}
  \end{array}\right] 
\end{align}

then it is straightforward to confirm that the drift matrix $L = AB - I$ satisfies the Lyapunov equation

\begin{align}
    L\omega + \omega L^T + I = 0
\end{align}

Now, consider any skew-symmetric matrix $Q$ that is only supported in the top left $r \times r$ block.  Then $AB + Q \omega^{-1}$ will only be supported in the top left block and therefore still have rank $r$.  Furthermore, $Q$ obeys the trivial Lyapunov equation $Q\omega + \omega Q^T = 0$, so we have that the drift matrix $\hat{L} := AB + Q \omega^{-1} - I$ also satisfies the Lyapunov equation $\hat{L}\omega + \omega \hat{L}^T + I = 0$.  Choosing the norm of $Q$ small enough guarantees that $\hat{L}$ is still Hurwitz, while choosing $Q \neq 0$ guarantees $L \neq \hat{L}$.

In the worst case, every intervention is of the form $e_i$ for $i > r$.  The block structure and Woodbury matrix identity imply that $(AB + Q \omega^{-1} - I)^{-1}e_i = -e_i$ for any $Q$ chosen as above, and therefore $-\hat{L}^{-1}e_i = -L^{-1}e_i$.

We conclude that for each intervention $e_i$ with $i>r$, by Theorem~\ref{thm:sarkka-linear} the drift matrices $L$ and $\hat{L}$ both induce identical stationary distributions.  Thus, even with $n-r$ interventions, the drift is not identifiable.

\subsection*{Proof of Theorem~\ref{thm: linear_tight}} \label{sec:linear_tight}

We show separately the upper and lower bounds for number of interventions needed for almost sure identifiability.

\begin{theorem}
    Suppose the true $A$ and $B$ are sampled according to Assumption~\ref{ass:sampling}, with the constraint that a.s. $\|A\|, \|B\| \leq \sqrt{\gamma}$.  Then $L = AB - D$ is a.s. identifiable.
\end{theorem}

\begin{proof}

Recall that $C = \left[c_1, \dots, c_r \right]$ is the matrix of interventions as columns. The means of the interventional stationary distributions are given by,
\begin{align}
    -L^{-1}C = (D^{-1} + D^{-1}A(I - BD^{-1}A)^{-1}BD^{-1})C
\end{align}
So with knowledge of $D$ and $C$, we can calculate  $A(I - BD^{-1}A)^{-1}BD^{-1}C$. This matrix is a.s. full rank, and its range is equal to the range of $A$.  Hence we can infer $P := P_A^\perp$. 

Now, using the Lyapunov equation, we have,
\begin{align}
    0 &= L\omega P + \omega L^T P + \epsilon P \\
    &= L\omega P - \omega D P + \epsilon P
\end{align}
Rearranging,
\begin{align}
    \omega P &= L^{-1}(\omega D - \epsilon I)P
\end{align}
and $(\omega D - \epsilon I)P$ is rank $n-r$, so a.s. $\im (\omega D - \epsilon I)P \oplus \im C$ is $n$ dimensional.  Hence we recover $L^{-1}$ and therefore $L$.

\end{proof}

\begin{theorem}
    Sample $A \in \mathbb{R}^{n \times r}$ and $B \in \mathbb{R}^{r \times n}$ according to Assumption~\ref{ass:sampling}.  Then a.s. $AB$ is not identifiable with $r-2$ interventions.
\end{theorem}
\begin{proof}
    Assume the decay is fixed at $D = I$, so $L = AB - I$ with induced covariance $\omega$.  With $r-2$ interventions, $ABL^{-1}C$ is at most rank $r-2$, and by assumption the kernel of $A^T$ is at most dimension $n-r$, so there is guaranteed to be a two dimensional subspace orthogonal to the previous spaces, say with basis vectors $u$ and $v$.  Let $Q^* = uv^T - vu^T$.

    Then we consider $\hat{L} = L + \omega Q$ where $Q = B^TA^TQ^*AB$, and claim that $\hat{L}$ is a valid, distinct drift matrix that generates the same data.

    First, note that $u, v$ are orthogonal to the kernel of $A^T$ and $B^T$ is full-rank a.s., so $B^TA^TQ^*$ is non-zero.  Transposing and applying the same reasoning, we get that $Q = B^TA^TQ^*AB \neq 0$.  Scaling down the magnitude of $Q^*$ if necessary, we have that $\hat{L}$ is Hurwitz and distinct from $L$.

    Second, note that $\hat{L} = (I + \omega B^TA^TQ^*)AB - I$, so it satisfies the rank $r$ constraint on the non-diagonal part.

    Now we show it agrees on the induced stationary distributions.  Again from the choice of the two dimensional subspace, we have that $Q^* AB L^{-1}C = 0$, which implies $QL^{-1}C = 0$.
    
    From the Woodbury matrix identity, we see the means of all stationary distributions are conserved,
    \begin{align}
        \hat{L}^{-1}C &= (L+\omega Q)^{-1}C \\
        &= \left[L^{-1} - L^{-1}\omega(I + QL^{-1}\omega)^{-1}QL^{-1} \right]C \\
        &= L^{-1}C
    \end{align}
    Finally, note that by antisymmetry of $Q$,
    \begin{align}
        \hat{L}\omega + \omega \hat{L}^T &= L\omega + \omega L^T + \omega Q \omega + \omega Q^T \omega \\
        &= L\omega + \omega L^T
    \end{align}
    which implies $\hat{L}$ satisfies the same Lyapunov equation and therefore induces the same covariance $\omega$.
\end{proof}

\subsection*{Proof of Proposition~\ref{prop:ode_failure}} \label{sec:ode_failure}
    In the zero-noise setting, the SDE reduces to an ODE, and in the limit as time goes to infinity, the stationary distribution is simply a point mass on the unique fixed point of the drift. Therefore, one only gets access to the stationary points $x_i^*$ for each vector field $v + c_i$: other higher moments are not defined. 

    Suppose we have $n-r-1$ interventions. Since $A$ has $r$ columns, there must be a vector $u$ in $\mathbb{R}^n$ orthogonal to the columns of $A$ and $C$.

    Let $\hat{B} = B + bu^T$ for any non-trivial $b \in \mathbb{R}^r$.  We claim $\hat{B}$ induces the same data as $B$.  First, note that because $u$ is orthogonal to every $c_i$, and $x_i^* - c_i$ is in the image of $A$, it follows $u$ is also orthogonal to each $x_i^*$.

    Thus, one can confirm that $\hat{B}x_i^* = Bx_i^*$, so the fixed points don't change.  But $\hat{B} \neq B$, and they are clearly not equivalent up to permutation or scaling invariance.

\subsection*{Proof of Theorem~\ref{thm:mean_cov}}
\label{sec:mean_cov}
We will require two lemmas.

\begin{lemma}\label{lem:norm_bound}
    Consider a vector field $v(x) = A\sigma(Bx) - x + c$ with $A$, $B^T$ in the Stiefel manifold of shape $n \times r$, $\|\sigma'\|_\infty = \gamma < 1$. Let $p$ be the stationary distribution of our usual SDE, $x^*$ be the unique stationary point of $v$, and $P$ be an orthogonal projection such that $BP = B$ and $PA = A$.  Assume $2(j-1) \leq Tr(P)$, then,
    \begin{align}
        E_p\left[\|P(x-x^*)\|^{2j}\right] &\leq \left(\frac{\epsilon Tr(P)}{1-\gamma}\right)^j.
    \end{align}
\end{lemma}
\begin{proof}
    Let $f(x) = A\sigma(Bx) + c$.  Note $f$ is a contraction and $f(x) = f(Px)$.  We have by Cauchy-Schwarz,
    \begin{align}
        P(x-x^*) \cdot v(x) &= P(x-x^*) \cdot (v(x) - v(x^*)) \\
        &= P(x-x^*) \cdot (f(x) - f(x^*) - (x - x^*)) \\
        &= P(x-x^*) \cdot (f(Px) - f(Px^*) - (x - x^*)) \\
        &\leq -(1-\gamma)\|P(x-x^*)\|^2
    \end{align}

    Choose $g_j(x) = \|P(x-x^*)\|^{2j}$, then if $T := Tr(P)$ we have,
    \begin{align}
        \nabla g_j(x) &= 2j g_{j-1}(x) P(x-x^*)\\
        \Delta g_j(x) &= 2jT g_{j-1}(x) + 4j(j-1)g_{j-1}(x)
    \end{align}
The generator of the SDE is $\mathcal{A}g = \nabla g \cdot v + \frac{\epsilon}{2} \Delta g$.  So Fokker-Planck gives,
    \begin{align}
        0 &= E_p \left[\mathcal{A}g_j \right] \\
        &= E_p \left[\nabla g_j \cdot v(x) + \frac{\epsilon}{2} \Delta g_j \right] \\
        &= E_p \left[2j g_{j-1} P(x-x^*) \cdot v(x) + \frac{\epsilon}{2}(2jT + 4j(j-1)) g_{j-1}(x)\right] \\
        &\leq E_p \left[-2j(1-\gamma)g_j + \frac{\epsilon}{2}(2jT + 4j(j-1)) g_{j-1}(x)\right] \\
    \end{align}
    Simple algebra gives,
    \begin{align}
        E_p \left[g_j\right] &\leq \frac{\epsilon(T + 2(j-1))}{2(1-\gamma)}E_p \left[ g_{j-1}(x)\right]
    \end{align}
    Now apply the assumption $2(j-1) \leq T$ and the result follows by induction.
\end{proof}

\begin{lemma}\label{lem:remainder}
    Assume same conditions as previous lemma.  Then the Taylor expansion with remainder around $x^*$ given by
    \begin{align}
        v(x) = Jv(x^*)(x-x^*) + R(x),
    \end{align}
    satisfies the bound
    \begin{align}
        \|R(x)\| \leq 2\sqrt{2}r^{3/2}M \|P(x-x^*)\|^2,
    \end{align}
    where $M := \sup_i \|\sigma_i''\|_\infty$.
\end{lemma}
\begin{proof}
Assume that $\im A \oplus (\ker B)^\perp$ is only supported on the first $2r$ coordinates. The $i$th coordinate of the Taylor remainder must take the form,
    \begin{align}
        R_i(x) &= v_i(x) - e_i^T Jv(x^*) (x-x^*) \\
        &= \sum_{|\alpha| = 2} \frac{2}{\alpha !} (x-x^*)^\alpha \int_0^1 (1-t) \partial_\alpha v_i(x^* + t(x-x^*)) dt
    \end{align}
By assumption, second order derivatives of $v$ are always bounded by $M$, and they are identically zero if $i > 2r$.  This implies $R_i$ is only nonzero for $i \leq 2r$, in which case,
    \begin{align}
        |R_i(x)| &\leq M  \sum_{|\alpha| = 2, \alpha_{>2r} = 0} \frac{2}{\alpha !} \left|(x-x^*)^\alpha \right| \\
        &\leq M \left(\sum_{j=1}^{2r} |x_j - x_j^*| \right)^2 \\
        &\leq 2r M \left(\sum_{j=1}^{2r} |x_j - x_j^*|^2 \right)  \\
        &= 2rM \|P(x-x^*)\|^2
    \end{align}
    where $P$ the projection onto the first $2r$ coordinates.

    Hence,
    \begin{align}
        \|R(x)\|^2 \leq 8r^3 M^2 \|P(x-x^*)\|^4.
    \end{align}
    
Now, we drop the assumption that $\im A \oplus (\ker B)^\perp$ is restricted to the first $2r$ elements and extend the result more generally.
    
    Consider the orthogonal matrix $Q$ such that $\im QA \oplus (\ker BQ^T)^\perp$ is restricted to the first $2r$ components, then by design the above result can be applied to the new drift function
    
    \begin{align}
        \tilde{v}(y) := QA\sigma(BQ^Ty) - y + Qc
    \end{align} 
    
    whose unique stationary point is $y^* = Qx^*$.  

    Hence, applying the above result to the remainder term $\tilde{R}(y) := \tilde{v}(y) - J\tilde{v}(y^*)(y - y^*)$, with $\tilde{P}$ the projection onto the first $2r$ components, we get

    \begin{align}
        \|\tilde{R}(y)\| \leq 2\sqrt{2}r^{3/2}M \|\tilde{P}(y-y^*)\|^2
    \end{align}

    If we let $y = Qx$, then some algebra reveals
    \begin{align}
        \tilde{R}(y) &= \tilde{v}(y) - J\tilde{v}(y^*)(y - y^*) \\
        &= QA\sigma(BQ^Ty) - y + Qc - (QAJ\sigma(BQ^Ty^*)BQ^T - I)(y-y^*) \\
        &= Q(A\sigma(Bx) - x + c - (AJ\sigma(Bx^*)B - I)(x-x^*)) \\
        &= Q(v(x) - Jv(x^*)(x - x^*)) \\
        &= QR(x)
    \end{align}

    And so we can finally conclude

    \begin{align}
        \|R(x)\| &= \|\tilde{R}(y)\| \\
        &\leq 2\sqrt{2}r^{3/2}M \|\tilde{P}(y-y^*)\|^2 \\
        &= 2\sqrt{2}r^{3/2}M \|Q^T\tilde{P}Q(x-x^*)\|^2
    \end{align}

    Note that by our choice of $Q$, $Q^T\tilde{P}Q$ is the orthogonal projection onto $\im A \oplus (\ker B)^\perp$, and so we have the bound.
    
\end{proof}

Using these results, we can proceed to the main perturbation theory result:

\begin{proof}[Proof of Theorem~\ref{thm:mean_cov}]
    We will drop the $\epsilon$ subscript, where it is clear that $m:=m_\epsilon$ and $\Sigma:=\Sigma_\epsilon$ are the first and second order moments of the stationary distribution.

    The mean bound follows from Lemma~\ref{lem:norm_bound} and Jensen's inequality,
    \begin{align}
        \|m - x^*\| &= \|E_p[x - x^*]\| \\
        &\leq E_p\left[\left\|x-x^* \right\|^2 \right]^{1/2} \\
        &\leq \left(\frac{\epsilon n}{1-\gamma}\right)^{1/2}.
    \end{align}
    Fokker-Planck yields a formula for second-order moments of the SDE stationary distribution (see for example Chapter 5.5 in~\citet{sarkka2019applied}), such that,
\begin{align}
    0 &= E_{p}[(x-m)v(x)^T + v(x)(x-m)^T + \epsilon I]
\end{align}

Note the simple fact that
\begin{align}
    E_p[(x-m)(x-x^*)^T] &= E_p[(x-m)(x-m + m -x^*)^T] \\
    &= E_p[(x-m)(x-m)^T] = \Sigma
\end{align}
combined with the linearization of $v$ we can write,
\begin{align}
    0 &= E_{p}[(x-m)(x-x^*)^TL^T + L(x-x^*)(x-m)^T + \epsilon I + (x-m)R(x)^T + R(x)(x-m)^T] \\
    &=L\Sigma + \Sigma L^T + \epsilon I + E_{p}[(x-m)R(x)^T + R(x)(x-m)^T]
\end{align}
and dividing by $\epsilon$, 
\begin{equation}\label{eq:lyapunov-approx}
    0 = L\frac{\Sigma}{\epsilon} + \frac{\Sigma}{\epsilon} L^T + I + \frac{1}{\epsilon}E_{p}[(x-m)R(x)^T + R(x)(x-m)^T]
\end{equation}

As in the lemmas, we let $P$ be the orthogonal projection onto $\im A \oplus (\ker B)^\perp$.  Then
 by Lemma~\ref{lem:norm_bound}, Lemma~\ref{lem:remainder} and Cauchy-Schwarz we have,
\begin{align}
    E_p\left[\|x-m\| \cdot \|R(x)\| \right] &\lesssim r^{3/2} M E_p\left[(\|x-m\| \cdot \|P(x-x^*)\|^2 \right] \\
    &\lesssim r^{3/2} M E_p\left[\|x-m\|^2\right]^{1/2} E_p\left[ \|P(x-x^*)\|^4 \right]^{1/2} \\
    &\lesssim \frac{\epsilon r^{5/2} M}{1-\gamma}  E_p\left[\|x-m\|^2\right]^{1/2} \\
    &\lesssim \frac{\epsilon r^{5/2} M}{1-\gamma}  E_p\left[\|x-x^*\|^2 + \|x^* - m\|^2\right]^{1/2} \\
    &\lesssim \frac{\epsilon^{3/2} r^{5/2} n^{1/2} M}{(1-\gamma)^2}
\end{align}

By the fact that $\|xy^T\| \leq \|x\|\cdot\|y\|$ and Jensen's inequality,
\begin{align}
    \left\|\frac{1}{\epsilon}E_{p}\left[(x-m)R(x)^T + R(x)(x-m)^T\right]\right\|
    &\lesssim \frac{1}{\epsilon} E_{p}\left[\|x-m\| \cdot \|R(x)\|\right] \\
    &\lesssim \frac{\epsilon^{1/2} r^{5/2} n^{1/2} M}{(1-\gamma)^2}
\end{align}

Now, if we choose $\epsilon$ small enough to guarantee the above bound is strictly less than 1, then the Lyapunov equation given in Equation~\eqref{eq:lyapunov-approx} has a unique solution.  Furthermore, the Lyapunov equation has an integral form~\citep{sarkka2019applied}, which states that the unique positive definite matrix $\omega$ that satisfies $L\omega + \omega L^T + QQ^T = 0$ can be written as

\begin{align}
    \omega = \int_0^\infty e^{Lt} QQ^T e^{L^Tt} dt
\end{align}

It's easy to then conclude from the integral form of the Lyapunov equation in Equation~\eqref{eq:lyapunov-approx} that,
\begin{align}
    \|\Sigma/\epsilon - \omega\| \lesssim \frac{\epsilon^{1/2} r^{5/2} n^{1/2} M}{(1-\gamma)^2} \int_{0}^\infty \|e^{Lt}\| \|e^{L^Tt}\| dt.
\end{align}

By assumption, $L = A(J\sigma(x^*))B - I$.  It follows $\|e^{Lt}\| = \|e^{A(J\sigma(x^*))Bt}e^{-t}\|  \leq e^{-(1-\gamma)t}$ for all $t > 0$, and therefore
\begin{align}
    \|\Sigma/\epsilon - \omega\| &\lesssim \frac{\epsilon^{1/2} r^{5/2} n^{1/2} M}{(1-\gamma)^2} \int_{0}^\infty e^{-2(1-\gamma)t} dt \\
    &\lesssim \frac{\epsilon^{1/2} r^{5/2} n^{1/2} M}{(1-\gamma)^3}
\end{align}

\end{proof}

\subsection*{Proof of Theorem~\ref{thm:nonlinear_quant} and Theorem~\ref{thm:nonlinear_upper}}
We begin with a series of lemmas.  We first need a result about the equilibrium points of the drift, ruling out any degeneracies.  As before, we let $x_i^*$ denote the unique zero of $v(x) + c_i = A\sigma(Bx) -x + c_i$.  Additionally, in the sequel we will consider $\lesssim$ to absorb all constants $\gamma, \tau, M, \nu$.

\begin{lemma}\label{lem:approx}
    For $A$ and $B^T$ of shape $n \times r$ sampled from the Stiefel manifold, for $n$ sufficiently larger than $r$, with probability $1 - 2\exp(-\tilde{C}n^{2/3})$ we have $s_r(P_A^\perp B^T) \geq 1 - n^{-1/3}$.
\end{lemma}
\begin{proof}
    By rotational invariance and the Gaussian sampling of the Haar measure, it is known that $s_r(P_A^\perp B^T)^2$ is distributed as $1 - \|(G^TG)^{-1/2}G_1^TG_1(G^TG)^{-1/2}\|$
with $G = \bigl(\begin{smallmatrix} G_1 \\ G_2 \end{smallmatrix}\bigr)$ having iid Gaussian entries, $G_1$ of shape $r \times r$ and $G_2$ of shape $(n-r) \times r$.

By tight bounds on the singular values of Gaussian matrices~\citep{vershynin2025high}, with probability at least $1 - 2\exp(-\tilde{C}n^{2/3})$ we have the bounds $s_1(G_1) \leq 2\sqrt{r} + n^{1/3}$ and $s_r(G) \geq \sqrt{n} - \sqrt{r} - n^{1/3}$.  The bound follows.

\end{proof}

\begin{lemma}\label{lem:fixed_point_alpha_assume}
    Assume $k \gtrsim r^2$ and $n \gtrsim r^2 k^3$.  If $X \in \mathbb{R}^{n \times k}$ is defined by $X_i := x_i^* - c_i$, then with probability at least $1 - 2\exp(-\tilde{C}k^{2/3}) - 2\exp(-n^{1/3})$, $\sqrt{k} \lesssim s_r(X) \leq s_1(X) \lesssim \sqrt{k}$.
\end{lemma}
\begin{proof}

    Note that by rotational invariance, $BA$ is distributed the same as the principal $r \times r$ block of a Haar distributed matrix on the orthogonal group $O(n)$, or equivalently $SO(n)$ since the marginal for the $r \times r$ block is the same.  Hence, by concentration of measure for Haar measures, specifically Theorem 5.17 in~\citet{meckes2019random}

    \begin{align}
        P(\|BA\| \geq E\left[\|BA\| \right] + t) \leq \exp(-\tilde{C}(n-2)t^2)
    \end{align}

    Now, using the fact that the marginal distribution of each row of $B$ is uniform on the $n-1$ dimensional sphere, we note that

    \begin{align}
        (E\|BA\|)^2 &\leq E\|BA\|^2 \\
        &\leq E\|BA\|_F^2 \\
        &=r^2 E[\langle b_1, e_1 \rangle^2]\\
        &= r^2 / n
    \end{align}

    Hence, by setting $t = (r^2/n)^{1/3}$ we have with probability at least $1 - \exp(-\tilde{C}n^{1/3}r^{4/3})$ that $\|BA\| \leq 2(r^2/n)^{1/3}$.  For the rest of the argument, we condition on the values of $A$ and $B$ satisfying this event.

    Let $g(y, z) = \sigma(BAy + z)$.  By the contractive constraint on the activation, and the constraints that $\|A\|, \|B\| \leq 1$, it's clear $g$ is a uniform contraction mapping. The derivatives of $g$ satisfy,
    \begin{align}
        J_y g(y,z) &= J\sigma(BAy + z) BA \\
        J_z g(y,z) &= J\sigma(BAy + z)
    \end{align}
    with all spectral norms bounded by $\gamma < 1$.

    By the uniform contraction mapping principle, the fixed point map $y^*(z)$ such that $g(y^*(z), z) = y^*(z)$ is differentiable, with Jacobian
    \begin{align}
        Jy^*(z) = (I - J_y g(y^*(z), z))^{-1} J_z g(y^*(z), z)
    \end{align}

    Notice this Jacobian never vanishes, 
    and furthermore $\|y^*(z)\| \rightarrow \infty$ as $\|z\| \rightarrow \infty$.  This follows from the fact that $\sigma' \geq \tau$,
    \begin{align}
        \|y^*(z)\| &= \|\sigma(BAy^*(z) + z)\| \\
        &\geq \|\sigma(0) + \tau (BAy^*(z) + z)\|
    \end{align}
    which is clearly impossible for $\|z\|$ sufficiently large and $\|y^*\|$ bounded.

    Therefore, by the Hadamard global inverse theorem, $y^*(z)$ is a diffeomorphism on $\mathbb{R}^r$.

    Next, we can show that $y^*$ is bi-Lipschitz.  Namely:

    \begin{align}
        \|y^*(z_1) - y^*(z_2)\| &= \left\|\int_0^1 Jy^*(z_2 + t(z_1 - z_2))(z_1 - z_2) \right\| \\
        &\leq \frac{1}{1-\gamma} \|z_1 - z_2\|
    \end{align}

    and 

    \begin{align}
        \|(y^*)^{-1}(z_1) - (y^*)^{-1}(z_2)\| &= \left\|\int_0^1 J(y^*)^{-1}(z_2 + t(z_1 - z_2))(z_1 - z_2) \right\| \\
        &\leq \left(\frac{1}{\tau} + 1\right)\|z_1 - z_2\|
    \end{align}

    Replacing $z_i$ with $y^*(z_i)$ in the second inequality, and putting these statements together yields

    \begin{align}
        \frac{\tau}{2} \|z_1 - z_2\| \leq \|y^*(z_1) - y^*(z_2)\| \leq \frac{1}{1-\gamma} \|z_1 - z_2\|
    \end{align}

    Now, from the assumption that $\sigma$ is odd, it follows that $y^*(0) = 0$ and we therefore have

    \begin{align}
        \frac{\tau}{2} \|z\| \leq \|y^*(z)\| \leq \frac{1}{1-\gamma} \|z\|
    \end{align}

    Now, consider matrices $Y = [y^*(Bc_1), \dots, y^*(Bc_{k})]$ and $\tilde{Y} = [\sigma(Bc_1), \dots, \sigma(Bc_{k})]$.  Note that $Bc_i$ is distributed according to an isotropic Gaussian in $\mathbb{R}^r$, so $\tilde{Y}$ has columns that are mean zero and constant subgaussian norm because $\sigma$ is odd and has a constant Lipschitz norm.  Furthermore, by independence of cross terms and the fact $|\tau g| \leq |\sigma(g)| \leq |\gamma g|$,

    \begin{align}
        \tau^2 I \preceq E\left[\sigma(Bc)\sigma(Bc)^T\right] \preceq \gamma^2 I
    \end{align}

    Thus, by Theorem 4.6.1 in~\citet{vershynin2025high}, if $k \gtrsim r^2$ then with probability at least $1 - 2\exp(-k^{2/3})$:

    \begin{align}
        \sqrt{k} \lesssim s_{r}(\tilde{Y}) \leq s_1(\tilde{Y}) \lesssim \sqrt{k}
    \end{align}

    By our condition on $\|BA\|$ we can write

    \begin{align}
        \|y^*(Bc) - \sigma(Bc)\| &= \|\sigma(BAy^*(Bc) + Bc) - \sigma(Bc)\| \\
        &\leq \gamma \|BA\| \|y^*(Bc)\| \\
        &\lesssim (r^2/n)^{1/3} \|Bc\|
    \end{align}

    Because each $Bc_i$ is an isotropic Gaussian in dimension $r$, by Theorem 3.1.1 in~\citet{vershynin2025high}, we have each vector satisfies $\|Bc_i\| \lesssim \sqrt{r} + \sqrt{k}$ with probability at least $1 - 2k\exp(-\tilde{C}k)$.

    Hence, $\|Y-\tilde{Y}\| \lesssim k(r^2/n)^{1/3}$, and for sufficiently large $n$ we have by Weyl's inequality

    \begin{align}
        \sqrt{k} \lesssim s_{r}(Y) \leq s_1(Y) \lesssim \sqrt{k}
    \end{align}

    Finally, because $X = AY$ and $A$ has orthonormal columns, we can inherit the same bounds.

\end{proof}
\begin{lemma}\label{lem:fixed_point_beta}
    Conditioning on $A$, $B$, the vector $\sigma'(Bx_0^*)$ has a density with respect to the Lebesgue measure.
\end{lemma}
\begin{proof}

    Again we first condition on $A$ and $B$.  Using notation from Lemma~\ref{lem:fixed_point_alpha_assume} we have that $\sigma(Bx_i^*) = y^*(Bc_i)$, and therefore since $\sigma$ is smooth and strictly monotonic, we have that each $Bx_i^* = \sigma^{-1}(y^*(Bc_i))$ has a density and is independent of the other interventions.

    Define $f(z) = \sigma'(z)$. It's quick to see that $Jf(z)$ is diagonal for any $z$, and rank deficient only when there's some index $j$ such that $\sigma_j''(z)$ is zero.  By assumption on $\sigma$, this set of points is measure zero.  Hence the Jacobian of $f$ is a.s. full-rank.  By the change of variable theorem induced by the area theorem~\citep{evans2025measure} we have that $f(Bx_i^*)$ has a density.
\end{proof}

\begin{lemma}\label{lem:fixed_point_beta_assume}
    Given the matrix with columns $\sigma'(Bx_0^*)/\sigma'(Bx_i^*) - \mathbf{1}$, with probability at least $1 - 2\exp(-\tilde{C}k^{1/3}) - 2\exp(-n^{1/3})$ choosing $k-1$ of these columns yields a matrix $X$ with $\sqrt{k} \lesssim s_{r}(X) \leq s_1(X) \lesssim \sqrt{kr}$.
\end{lemma}
\begin{proof}

    Condition on $A$ and $B$ holding in the same event as in Lemma~\ref{lem:fixed_point_alpha_assume} throughout the following argument.  Let $g$ denote iid isotropic Gaussian sampled in $r$ dimensions, and $z_i = Bc_i$ is notably also standard Gaussian.
    
    Let $w(g) = 1 / \sigma'(g) - E[1 / \sigma'(g)]$ and $w_i = w(z_i)$.
    
    Note that each $w_i$ is mean zero and is bounded almost surely by constant terms.  Conditioning on $w_0$, consider $\tilde{X}$ of shape $r \times {k-1}$ where $\tilde{X}_i = w_i$, and $\hat{X} = \tilde{X} - w_0\mathbf{1}^T$.  Conditioned on $w_0$, each column of $\hat{X}$ is independent, so $\hat{X}\hat{X}^T$ can be written as the sum of independent, bounded rank-one outer products.
    
    It is quick to confirm $\lambda_{min}(E[\hat{X}\hat{X}^T | w_0]) \gtrsim k$.  Therefore, by the fact that $s_r(\hat{X})^2 = \lambda_{min}(\hat{X}\hat{X}^T)$ and the Matrix Chernoff bound~\citep{tropp2012user}, it follows that with probability at least $1 - \exp(-\tilde{C} k^{1/3})$, $s_r(\hat{X}) \gtrsim \sqrt{k}$, and integrating over $w_0$ this is true without conditioning.  By a standard application of Theorem 4.6.1 in~\citet{vershynin2025high} and Weyl's inequality, there is an equivalent high probability guarantee that $s_1(\hat{X}) \lesssim \sqrt{kr}$.

    Define $X$ as the matrix with columns $X_i = 1/\sigma'(Bx_{i}^*) - 1/\sigma'(Bx_{0}^*)$. Again because $g \mapsto 1/\sigma'(g)$ is Lipschitz, we get the bounds

    \begin{align}
        \|X_i - (\tilde{X}_i - w_0)\| &= \|1/\sigma'(Bx_{i}^*) - 1/\sigma'(Bx_{0}^*) - (w_{i} - w_0)\| \\
        &\leq \|1/\sigma'(Bx_{i}^*) - 1/\sigma'(Bc_{i})\| + \|1/\sigma'(Bx_{0}^*) - 1/\sigma'(Bc_0)\| \\
        &\lesssim \|BAy^*(Bc_{i})\| + \|BAy^*(Bc_0)\| \\
        &\lesssim \|BA\| (\|Bc_{i}\| + \|Bc_0\|)
    \end{align}
    
    By an identical argument as in Lemma~\ref{lem:fixed_point_alpha_assume}, we therefore have $\|X - (\tilde{X} - w_0\mathbf{1}^T)\| \leq k(r^2/n)^{1/3}$ with probability at least $1 - 2k\exp(-\tilde{C}k)$.
    
    Finally, $diag(\sigma'(Bx_0^*)) X$ yields the desired matrix, and the bounds follow again by Weyl's inequality.

\end{proof}

The following lemma is essentially an application of the main result in~\citet{bhaskara2014uniqueness}, rephrased in our setting:

\begin{lemma}\label{lem:tensor-cp}[Theorem 2.6, 2.8 in~\citet{bhaskara2014uniqueness}]
    Consider matrices $X \in \mathbb{R}^{n_1 \times r}$, $D \in \mathbb{R}^{n_2 \times r}, Y \in \mathbb{R}^{n_3 \times r}$ where $r \leq \min(n_1, n_2, n_3)$, with $poly(n_1, n_2, n_3, r)$ bounds on all top and bottom singular values as well as column $l_2$ norms, and the tensor $\mathcal{T}$ with CP decomposition $[X,D,Y]$, i.e. $\mathcal{T} = \sum_{i=1}^r X_i \otimes D_i \otimes Y_i$.  Then given access to $\tilde{\mathcal{T}}$ such that $\|\mathcal{T} - \tilde{\mathcal{T}}\|_F \leq \epsilon$, there is an algorithm that infers $\tilde{X}, \tilde{D}, \tilde{Y}$ (all with polynomial bounds on $l_2$ norms), along with diagonal $\Lambda_X, \Lambda_D, \Lambda_Y$ and permutation $\Pi$ such that

    \begin{align*}
        \|\Lambda_X\Lambda_D\Lambda_Y - I\|_F &\leq poly(n_1, n_2, n_3, r)\epsilon \\
        \|\tilde{X} - X\Pi\Lambda_X\|_F &\leq poly(n_1, n_2, n_3, r)\epsilon \\
        \|\tilde{D} - D\Pi\Lambda_D\|_F &\leq poly(n_1, n_2, n_3, r)\epsilon \\
        \|\tilde{Y} - Y\Pi\Lambda_Y\|_F &\leq poly(n_1, n_2, n_3, r)\epsilon \\
    \end{align*}
\end{lemma}

Our last lemma requires some exposition.  Consider $A$ and $B$ in the support as defined in Assumption~\ref{ass:weak_nonlinear}, and diagonal $D \in \mathbb{R}^{r \times r}$ in the support of $\sigma'(\mathbb{R}^r)$.  Let $\omega$ be the solution to the Lyapunov equation $(ADB - I)\omega + \omega(ADB-I)^T + I = 0$.  Define $\Gamma = A^T \omega^{-1} A$.

Given diagonal matrices with free parameters $\alpha, \beta \in \mathbb{R}^{r \times r}$, define $X = \alpha (DBA - diag(DBA)) + \beta$.

\begin{lemma}\label{lem:poly}
    There is some set of $2r$ symmetric matrices $S_1, \dots, S_{2r}$ such that if $g(A, B, D)$ is the determinant of the $2r \times 2r$ matrix acting on $\alpha$ and $\beta$ in the system of equations $Tr(S_i \Gamma X) = Tr(S_i \Gamma DBA)$, $g$ is not identically zero.
\end{lemma}
\begin{proof}

We will assume for simplicity that $r$ is divisible by 3, a similar argument works in other cases.  We simply need to find a witness in the support of the inputs that realizes a non-zero value for $g$.

Let $A$ be a block matrix $\left[
\begin{array}{c}
I \\ \hline
0
\end{array}\right]$ so that $A^T A = I$.  Let $D = \delta^{-1} I$ for any $\delta$ satisfying $\tau \leq \delta^{-1} \leq \gamma$ and define $S^*$ as a block diagonal matrix where each $3 \times 3$ block is of the form $
\frac{1}{3\delta}\begin{bmatrix}
1 & -2 & -2 \\
-2 & 1 & -2 \\
-2 & -2 & 1
\end{bmatrix}
$.

We choose $B = D^{-1}S^* A^T$, and observe that $BB^T = D^{-1}(S^*)^2D^{-1} = I$ so $B^T$ is also supported on the Stiefel manifold.

Because $ADB - I = AS^*A^T - I$, this implies $\omega = -\frac{1}{2}(AS^*A^T - I)^{-1}$ and $\Gamma = -2A^T(AS^*A^T - I)A = -2(S^* - I)$.

Additionally, because $DBA = S^*$, we have 

\begin{align}
    X &= \alpha VA + \beta \\
    &= \alpha (DBA - diag(DBA)) + \beta \\
    &= \alpha(S^* - diag(S^*)) + \beta \\
\end{align}

Thus, we can write 

\begin{align}
    -\frac{1}{2}Tr(S_i \Gamma X) &= Tr(S_i (S^* - I)\alpha (S^* - \frac{1}{3\delta} I)) + Tr(S_i (S^* - I) \beta)
\end{align}

Now, if $E_{i,j}$ denotes the one-hot matrix with a single one at coordinate $(i,j)$, consider the set of symmetric matrices $\{E_{11}, E_{22}, E_{33}, E_{12} + E_{21}, E_{13} + E_{31}, E_{23} + E_{32}\}$ and consider $S_1, \dots, S_6$ also block diagonal with each of these matrices as the first $3 \times 3$ block.  It follows from straightforward algebra that, ignoring the $-\frac{1}{2}$ scaling term, the matrix acting on $[\alpha_1, \alpha_2, \alpha_3, \beta_1, \beta_2, \beta_3]$ can be written as

\begin{align}
\begin{pmatrix}
0 & 4d^2 & 4d^2 & d-1 & 0 & 0 \\
4d^2 & 0 & 4d^2 & 0 & d-1 & 0 \\
4d^2 & 4d^2 & 0 & 0 & 0 & d-1 \\
-2d(d-1) & -2d(d-1) & 8d^2 & -2d & -2d & 0 \\
-2d(d-1) & 8d^2 & -2d(d-1) & -2d & 0 & -2d \\
8d^2 & -2d(d-1) & -2d(d-1) & 0 & -2d & -2d
\end{pmatrix}
\end{align}

where $d = 1/3\delta$.  The determinant of this matrix can be confirmed to be $16d^3(3d-1)^5(3d+1)$, and therefore has nonzero determinant when $\delta > 1$.  Applying an identical argument to every $3 \times 3$ block, and the fact that the total determinant is the product of the block determinants, shows that this is a valid witness that $g$ is non-zero.

\end{proof}

We introduce the notation $\sigma_{[i]}'$ to denote the Jacobian of $\sigma$ evaluated at $Bx_i^*$, to distinguish it from an element of $\sigma$.  Because $\sigma$ acts elementwise, $\sigma_{[i]}'$ is necessarily a diagonal matrix.  

Finally, for ease of notation in the perturbation analysis, given matrices $M_1, M_2$ we will write $M_1 = M_2 + O(\epsilon)$ to mean $\|M_1 - M_2\| = O(poly(n,r,k)\epsilon)$.  We use similar notation for vectors in the $l_2$ norm.

\begin{proof}[Proof of Theorem~\ref{thm:nonlinear_quant}]

By Theorem~\ref{thm:mean_cov}, for each intervention we approximately recover the mean vectors $m_i = x_i^* + O(\epsilon^{1/2})$ and the covariance matrices $\Sigma_i / \epsilon = \omega_i + O(\epsilon^{1/2})$, which satisfies $Jv(x_i^*)\omega_i + \omega_i (Jv(x_i^*))^T + I = 0$.  Note that the Jacobian of the vector field satisfies $Jv(x_i^*) = A\sigma_{[i]}'B - I$.

First, we observe that $\omega_i$ is well-conditioned.  We have from the integral formula for the Lyapunov equation that $\omega = \int_{0}^\infty e^{(A\sigma_{[i]}'B - I)t} e^{(A\sigma_{[i]}'B - I)^{T}t} dt$, so from the variational characterization of singular values we have $1 \lesssim s_n(\omega) \leq s_1(\omega) \lesssim 1$.  Hence, all of the matrices in the problem are well-conditioned, so we can freely apply matrix multiplications while the error stays on order of $O(\epsilon^{1/2})$.

From Lemma~\ref{lem:fixed_point_alpha_assume}, the matrix $X \in \mathbb{R}^{n \times k}$ with columns $X_i := x_i^* - c_i$ is rank-$r$ but well-conditioned.  If we define $\tilde{X}_i = m_i - c_i$, and $\hat{X}$ as the truncated SVD to rank $r$ of $\tilde{X}$, it follows 

\begin{align}
\|X - \hat{X}\| &\leq \|X - \tilde{X}\| + \|\tilde{X} -\hat{X}\| \\
&\leq 2 \|X - \tilde{X}\| \\
&= O(\epsilon^{1/2})
\end{align}

where the second inequality is because $\hat{X}$ is the best rank-r approximation in operator norm to $\tilde{X}$.  Hence by the Davis-Kahan theorem we can derive a projection matrix $\tilde{P}$ such that $\tilde{P} = P_A + O(\epsilon^{1/2})$.

Expanding the Lyapunov equation and applying $P_A^\perp$ on the right, and using the fact that all associated matrices have a constant upper bound on their operator norm, we obtain

\begin{align}
    A\sigma_{[i]}'B\frac{\Sigma_i}{\epsilon} \tilde{P}^\perp - 2\frac{\Sigma_i}{\epsilon} \tilde{P}^\perp + \tilde{P}^\perp = O(\epsilon^{1/2})
\end{align}

Rearranging, this yields
\begin{equation}\label{eq:projected_lyapunov}
    2\frac{\Sigma_i}{\epsilon} \tilde{P}^\perp = (I - \frac{1}{2}A\sigma_{[i]}'B)^{-1}\tilde{P}^\perp + O(\epsilon^{1/2}),
\end{equation}

Taking the transpose, 
\begin{align}
    2\tilde{P}^\perp \frac{\Sigma_i}{\epsilon} = \tilde{P}^\perp(I - \frac{1}{2}B^T\sigma_{[i]}'A^T)^{-1} + O(\epsilon^{1/2}).
\end{align}
The LHS matrix is rank $n-r$, so its kernel is dimension $r$ and we choose $Z_i \in \mathbb{R}^{n \times r}$ to have orthonormal columns that span this kernel.  Again, because all matrices are well-conditioned, it immediately follows from Davis-Kahan that
\begin{align}
    Z_i = (I - \frac{1}{2}B^T\sigma_{[i]}'A^T)AQ_i + O(\epsilon^{1/2}),
\end{align}
for some invertible matrix $Q_i \in \mathbb{R}^{r \times r}$ that is also well-conditioned.  Note that here, $Z_i$ is observed but $Q_i$ is not.

Because $\|\tilde{P}^\perp AQ_i\| \leq \|P_A^\perp - \tilde{P}^\perp\| \|AQ_i\| = O(\epsilon^{1/2})$,
\begin{align}
    \tilde{P}^\perp Z_i = -\frac{1}{2}\tilde{P}^\perp B^T\sigma_{[i]}'Q_i + O(\epsilon^{1/2}).
\end{align}
Observe two facts: 1) $\sigma_{[i]}'Q_i$ is a.s. an invertible matrix in $\mathbb{R}^{r \times r}$, and 2) $\tilde{P}^\perp B^T$ maps from $\mathbb{R}^r$ to an $n-r$ subspace, so if $n > 2r$ then a.s. this is full-rank.  Indeed, $\|\tilde{P}^\perp B - P_A^\perp B\| \lesssim O(\epsilon^{1/2})$ so $s_r(\tilde{P}^\perp Z_i) \gtrsim s_r(P_A^\perp B) - \Omega(\epsilon^{1/2}) \gtrsim 1$ under the event in Lemma~\ref{lem:approx}. 

Define $M_i = (\tilde{P}^\perp Z_i)^\dag \tilde{P}^\perp Z_0$.  Then we may write, again using Davis-Kahan in the last line:

\begin{align}
    \|M_i - Q_i^{-1}(\sigma_{[i]}')^{-1}\sigma_{[0]}'Q_0\| &= \|(\tilde{P}^\perp Z_i)^\dag \tilde{P}^\perp Z_0 - (P_A^\perp Z_i)^\dag P_A^\perp Z_0\| + O(\epsilon^{1/2}) \\
    &\leq \|(\tilde{P}^\perp Z_i)^\dag \tilde{P}^\perp Z_0 - (P_A^\perp Z_i)^\dag \tilde{P}^\perp Z_0\| + O(\epsilon^{1/2}) \\
    &\leq O(\epsilon^{1/2})
\end{align}

Therefore we have the approximate decomposition

\begin{align}\label{eq:diagonal_decomp}
    Z_i M_i - Z_0 &= A \left(\sigma_{[0]}'(\sigma_{[i]}')^{-1} - I\right) Q_0 + O(\epsilon^{1/2})
\end{align}

By Lemma~\ref{lem:fixed_point_beta_assume}, with high probability, the matrix whose $i$th column is the diagonal part of $\left(\sigma_{[0]}'(\sigma_{[i]}')^{-1} - I\right)$ is well-conditioned.

Now, we build a 3-tensor from this object.  Let $\Sigma'$ be the matrix of shape $r \times (k-1)$ whose $i$th column is the diagonal part of the matrix $\sigma_{[0]}'(\sigma_{[i]}')^{-1} - I$.  Then the 3-tensor $\tilde{\mathcal{T}}$ whose $i$th slice is $Z_iM_i - Z_0$ is $O(\epsilon^{1/2})$ close in Frobenius norm to the 3-tensor $\mathcal{T}$ composed as the CP decomposition $[A, Q_0^T, \Sigma'^T]$.

Thus we can apply Lemma~\ref{lem:tensor-cp} and we approximately recover the tensor decomposition up to permutation $\Pi$ and rescaling matrices $\Lambda_A, \Lambda_Q$, $\Lambda_D$.  Since our drift terms $A$ and $B$ have an invariance to permuting the neurons, we assume WLOG that the recovered permutation is the identity.  Likewise, there is a sign invariance between $A$ and $B$ since $\sigma$ is odd, so we can assume WLOG that $\Lambda_A > 0$.  Because Lemma~\ref{lem:tensor-cp} guarantees the recovered approximations have constant $l_2$ bounds on columns, we can normalize each column of $\tilde{A}$ and therefore we recover $\tilde{A} = A + O(\epsilon^{1/2})$.

A further fact from Lemma~\ref{lem:tensor-cp} gives us $\tilde{Q} = \Lambda_Q Q + O(\epsilon^{1/2})$ where $\Lambda_Q$ is diagonal, and has constant upper and lower bounds.  We write $\Lambda = \Lambda_Q$ from now on.

We can observe,
\begin{align}
    (-2\tilde{P}^\perp Z_0)\tilde{Q}^{-1} &= \tilde{P}^\perp B^T \sigma_{[0]}' Q_0 (\Lambda Q_0)^{-1} + O(\epsilon^{1/2})\\
    &= P_A^\perp B^T \sigma_{[0]}' \Lambda^{-1} + O(\epsilon^{1/2})
\end{align}

Taking a transpose, we've observed $\Lambda^{-1} \sigma_{[0]}' B P_A^\perp + O(\epsilon^{1/2})$.  

Now, returning to the definition of $Z_0$:

\begin{align}
     (\tilde{A}^T Z_0 \tilde{Q}^{-1})^T = \Lambda^{-1} - \frac{1}{2} \Lambda^{-1} \sigma_{[0]}' BA + O(\epsilon^{1/2})
\end{align}

For finishing the proof of Theorem~\ref{thm:nonlinear_quant}, we will now use the fact that $\sigma$ is known.  Unrolling the identity around the means, we have $Bx_i^* = \sigma^{-1}\left(\tilde{A}^\dag(m_i - c_i)\right) + O(\epsilon^{1/2})$.  This implies that we observe $\Sigma'$ directly, and therefore we observe $\Lambda_D$ up to $O(\epsilon^{1/2})$ error, and hence $\Lambda_A$ and $\Lambda = \Lambda_Q$ up to the same error.  Thus, because $\Lambda$ is $poly(n,k,r)$ well-conditioned by Lemma~\ref{lem:tensor-cp}, by rescaling we observe $\Lambda^{-1} \sigma_{[0]}' BA + O(\epsilon^{1/2})$.  Applying $(\tilde{A}^T\tilde{A})^{-1}\tilde{A}^T$ to the right hand side, we observe $\Lambda^{-1} \sigma_{[0]}' BP_A + O(\epsilon^{1/2})$.  Finally, adding this to our previous term and using the fact that $P_A + P_A^\perp = I$, we recover $\Lambda^{-1} \sigma_{[0]}' B + O(\epsilon^{1/2})$.  Since we also approximately know the two diagonal matrices in front of $B$, we can finally approximate $B$ up to the same error.
\end{proof}

\begin{proof}[Proof of Theorem~\ref{thm:nonlinear_upper}]

We use the notation of the previous proof, and show that we can still approximately recover $B$ even without knowing $\sigma$.  We will now allow big $O$ notation to absorb all dependencies except for $\epsilon$.

Let $a_1, \dots, a_r, w_1, \dots w_{n-r}$ form an orthonormal basis where $a_1, \dots, a_r$ are the columns of $A$, and consider the event that every element of $Ba_i$ and $Bw_j$ is bounded away from zero by at least $O(\epsilon^{1/2})$, as $\epsilon$ gets smaller the probablility of this event clearly tends towards 1.

Conditional on that event, by taking ratios of elements within the same row, we can recover $\frac{b_i^Ta_j}{b_i^Ta_k} + O(\epsilon^{1/2})$ for $i \neq j \neq k$ and $\frac{b_i^Tw_j}{b_i^Ta_k} + O(\epsilon^{1/2})$ for $i \neq k$.  In other words, we have approximate ratios including all inner products except $b_i^Ta_i$ for all $i$.

We introduce some additional notation.  We define $X = \sigma_{[0]}'BA$.  Clearly there are some diagonal matrices $\alpha, \beta$ such that $X = \alpha (X - diag(X)) + \beta$.  From the above, we have an $O(\epsilon^{1/2})$ approximation of $\alpha^* (X - diag(X))$ where $\alpha^*$ is a diagonal matrix that row normalizes $(X - diag(X))$  Since $A$, $B$ are Haar-distributed and $\sigma_{[0]}'$ has a conditional density given $A$ and $B$, we are guaranteed that the absolute value of every entry of $\alpha^*$ lives in some interval depending on $(n, k, r)$ not including zero with arbitrarily high probability.

To finish the proof, we will show that the remaining information about $X$ can be extracted from the Lyapunov equation.  We rearrange the original Lyapunov equation to the form:

\begin{align}
    A\sigma_{[0]}'B\omega_0 + \omega_0(A\sigma_{[0]}'B)^T = 2\omega_0 - I + O(\epsilon^{1/2})
\end{align}

Multiplying by $\omega_0^{-1}$ on both sides, and multiplying by $A^T$ on the left and $A$ on the right yields:

\begin{align}
    (A^T \omega_0^{-1} A) X + X^T (A^T \omega_0^{-1} A) = 2 A^T \omega_0^{-1} A - A^T \omega_0^{-2} A + O(\epsilon^{1/2})
\end{align}

Finally, letting $\Gamma = \tilde{A}^T (\Sigma_0/\epsilon)^{-1} \tilde{A}$ and plugging in our approximations leads to

\begin{align}
    \Gamma X + X^T \Gamma = 2 \tilde{A}^T (\Sigma_0/\epsilon)^{-1} \tilde{A} - \tilde{A}^T (\Sigma_0/\epsilon)^{-2} \tilde{A} + O(\epsilon^{1/2})
\end{align}

Multiplying on the left by any symmetric matrix $S$ and taking a trace, we can therefore approximate $Tr(S\Gamma X)$.

For each $S_i$ as in Lemma~\ref{lem:poly}, $\mathrm{Tr}(S_i\Gamma X)$
is linear in $\alpha$ and $\beta$, and stacking the $2r$ measurements yields a
system with coefficient matrix given by $C_0 \left[\begin{array}{ c | c }
    \alpha^* & 0 \\
    \hline
    0 & I
  \end{array}\right]$ where $C_0$ converges to the the coefficient matrix in Lemma~\ref{lem:poly} as $\epsilon \rightarrow 0$.  Because every entry of $\alpha^*$ is bounded away from zero with high probability, we can focus solely on the conditioning of $C_0$.

By Lemma~\ref{lem:poly}, $g(A, B, D)$ is not trivially zero and analytic on its domain.  Conditioned on $A$ and $B$, $D=\sigma_{[0]}'$ has a density, and therefore $P(|g| > t) \rightarrow 1$ as $t \rightarrow 0^+$.  Therefore by continuity we can choose $\epsilon$ small enough to lower bound $|\det C_0|$.  By a trivial upper bound on $s_1(C_0)$, which may be exponential in $r$ but has no dependence on $\epsilon$, we are guaranteed a lower bound on the smallest singular value $s_{2r}(C_0)$, and by inverting the linear system we have an approximation $X + O(\epsilon^{1/2})$.

Finally, from an approximation of $X$, we can take row ratios to approximate $\langle b_i, a_i \rangle$, which with the ratios above, and the knowledge that $B$ has unit norm rows, gives an arbitrarily good approximation of $B$ up to scaling rows by $\pm 1$.

\end{proof}

%%%%%%%%%%%%%%%%%%%%%%%%%%%%%%%%
\section{Experimental Details}\label{sec:experimental_details_appendix}
\subsection{Linear SDE Recovery}\label{app:exp_det_lin}
For the true simulated linear SDE, $A$ and $B$ are sampled with iid Gaussian entries, and then normalized to have spectral norm equal to $0.9$.  The true decay matrix $D$ has each diagonal entry sampled iid from the uniform distribution on the interval $[1,2]$.  The model is trained with the loss given in Equation~\eqref{eq:lin_loss}, with the Adam optimizer~\citep{kingma2014adam} with an initial learning rate of $0.005$ and $3000$ iterations.  Additionally, due to the non-convexity of the objective, on each training instance we run 100 independent initializations and pick the one with the smallest training error.  This excludes the oversampled setting with $k=r * \log(n)$, where we get good performance with only 5 independent initializations.

The plotted mean and standard deviations in Figure~\ref{fig:linearsde} are based on 5 separate instantiations of the true model and fitting a new model on the stationary distribution.  Experiments were run on CPU.  All experiments (including those below) were run on a Linux system.

\subsection{Nonlinear SDE Recovery}\label{app:exp_det_nonlin_simple}

The true simulated $SDE$ has $A$ and $B^T$ sampled uniformly from the Stiefel manifold, and rescaled by $\gamma = 0.995$, with $\epsilon=10^{-5}$ and $\sigma(x) = 0.7x + 0.3\tanh(x)$.  We use ambient dimension $n=8$ and true low-rank dimension $r=2$.  To get training data, we sample 1000000 samples from each perturbed SDE, initialized at the stationary point with thinning of 300 and burnin of 100.  We then calculate the mean and covariance for use in the loss.

In training, the learned network is trained for 10000 iterations and learning rate $0.002$ also in Adam, using the loss $\mathcal{L}_{nonlin}$ introduced in Section~\ref{sec:non_linear_sde_recovery}.  The interventions are iid Gaussians with standard deviation of $0.5$.  We do 100 runs, choose the drift with smallest training loss, and measure the normalized Frobenius error against the true matrices $A$ and $B$ (after unit scaling and minimized over all column / row permutations, as the identifiability is only up to scaling and permutation), then report the mean and standard deviation across 10 independent instances.

\subsection{Synthetic Nonlinear SDE Generalization}\label{app:exp_det_nonlin}

To parameterize learnable activations, given an input $x \in \mathbb{R}^n$, we learn a function that acts elementwise on $x$ without necessarily being the same function on each element, i.e. $\sigma(x) = \left[\sigma_1(x_1), \dots, \sigma_n(x_n) \right] $.  We choose to parameterize each $\sigma_i$ as two-layer MLP using the actual sigmoid $\tilde\sigma$ as the activation function.  As a warm start, we also add $0.1\tilde\sigma(x)$.

We consider a fixed true SDE, with hidden dimension $r = 3$, where the rectangular matrices $A$ and $B^T$ have one down the diagonal and zero elsewhere, $D = I$, $\gamma = 0.98$, and the elementwise activation function is given by
\begin{align*}
    \sigma_1(x) &= 3\cos(3(x-0.5))\\
    \sigma_2(x) &= 2\sin(2(x+1.5)) - 1\\
    \sigma_3(x) &= \sigma_1(x)
\end{align*}
The hidden dimension of the activation MLPs is fixed at 20.

Each intervention is sampled as a Gaussian vector with mean zero and variance $0.1$ on each entry, and from each intervened SDE we collect 5000 samples.  To do this, we use the Euler-Maruyama scheme to solve the SDE, and use MCMC to draw approximately independent samples from the stationary distribution, using $dt = 0.01$, a thinning factor of $300$ and the first 500 samples treated as burn-in.  We use an radial basis function kernel to parameterize the KDS loss.  

We train with 10 interventions, and evaluate on 20 held-out interventions.  Training is done with AdamW~\citep{loshchilovdecoupled}, with initial learning rate $0.003$ and 50000 iterations.  The hyperparameter ranges considered are given in Table~\ref{tab:hyperparams}.  All experiments were run on CPU.

\begin{table}[h]
 \caption{Hyperparameter ranges considered for synthetic nonlinear SDE experiments.}
 \label{tab:hyperparams}
\centering
 \begin{tabular}{cc} 
 \hline
 Hyperparameter & Range \\
  \hline
  model hidden size $r$ & $\{4, 8, 16\}$ \\
  kernel bandwidth & $\{3,5,7\}$ \\
  $L_1$ weight regularization & $\{0, 10^{-5}, 10^{-4}\}$ \\
  \hline
 \end{tabular}
\end{table}

\subsection{Simulated GRN SDE Generalization}\label{app:exp_det_beeline}

We first give a summary of the simulator.  BoolODE is scRNA-seq data simulator that samples mRNA readouts via a stochastic differential equation (SDE). The underlying cyclic GRN is encoded in the SDE via a Hill function based approximation to a boolean logical circuit. $x_i$, $p_i$, and $R[i]$ each represents the mRNA concentration for gene $i$, the concentration of the protein encoded by gene $i$, and the set of proteins that regulate gene $i$. Further, $r$, $l_p$, $l_x$, and $s$ are scalar coefficients representing translation rate, RNA decay rate, protein decay rate, and noise standard deviation respectively. The functions $f_i(\cdot)$ encodes the gene regulatory network.  The default SDE is given by,
\begin{align}
    \frac{dx_i}{dt} &= f_i(p_{R[i]}) - l_x x_i + s \sqrt{x_i} dB_t \\
\frac{dp_i}{dt} &= r x_i - l_p p_i + s \sqrt{p_i} dB'_t
\end{align}
where $B$ and $B'$ are independent Brownian motions. 

We simulate overexpression (e.g., CRISPR-a) experiments by inducing perfect intervention coupled with increased transcription for a set of intervened genes $I_k$. For each intervened gene $j \in I_k$ we set $f_j(\cdot)=0$ and add a positive shift (set to 20 in the experiments) to $dx_j/dt$. We consider $K$ such intervention regimes, plus the observational setting ($k=0$) with no intervention, i.e, $I_0=\emptyset$, for a total of $K+1$ regimes. With this setup, we obtain a family of distributions $\{\rho_k\}_{k=0}^K$ in gene expression space. 

We adapt the neural ODE architecture proposed in~\citet{lin2025interpretable}.  The base SDE model under $I_k$ parametrizes the change of gene expression via the drift function as below. $A$ and $B$ are the coefficient matrices encoding the modular graph. $\alpha$ is a scaling vector controlling the rate of non-linear activation in the module, where the activation $\sigma$ is the logistic sigmoid.  Further, $\beta$ is a bias term that shifts the activation threshold of the modules. The GRN is extracted as $A \, \text{diag}(\alpha \circ \mathbf{1}_N) B$. The intervention is modeled as a combination of standard basis vectors, $\sum_{j \in I_k} e_{j}$, specifying the overexpression.  The model also learns a single diffusion coefficient scalar. Lastly, $M$ is a masking matrix blocking the signals from the intervened genes' regulators.  Altogether, the drift is given by,
\begin{align}
    v_k(x) = M_{k} A \sigma( \alpha \circ (B x - \beta) ) + \sum_{j \in I_k} e_{j} - Dx.
\end{align}
We alternatively consider the architecture using a learnable MLP to model the nonlinear activation of modular signals, 
\begin{align}
v_k(x) = M_{I_k} A \sigma_*  (B x - \beta)  + \sum_{j \in I_k} e_{j} - Dx
\end{align}
where $\sigma_*$ denotes the  MLP as described above.  The model loss is the Sinkhorn divergence~\citep{feydy2019interpolating} applied to the true perturbed distribution and the samples drawn from the learned SDE.

Each interventional distribution $\rho_k$ is obtained by taking the observational data $\rho_0$ as initial distribution and simulating the SDE via the Euler-Maruyama method.

We fit the hyperparameters by testing the ODE architecture in~\citet{lin2025interpretable} on the same simulated datasets, and then doing zero-shot transfer of the hyperparameters to the SDE models. For Bicycle, we sweep the hyperparameters in Table~\ref{tab:bi_hyperparams} and use the package defaults for the remaining values.   Experiments are run on an nvidia tesla V100 gpu.

\begin{table}[h]
 \caption{Hyperparameter ranges for GRN simulation.}
 \label{tab:bi_hyperparams}
\centering
 \begin{tabular}{cc} 
 \hline
 Hyperparameter & Range \\
  \hline
  scale $L_1$ & $\{0.0001, 0.001, 0.01, 0.1, 1.0\}$ \\
  scale spectral & $\{0, 1.0\}$ \\
  scale Lyapunov & $\{0.1, 1, 10\}$ \\
  \hline
 \end{tabular}
\end{table}

\end{document}